\documentclass[11pt]{article}
\pdfoutput=1
\usepackage{enumerate}
\usepackage{pdfsync}
\usepackage[OT1]{fontenc}
\usepackage{smile}
\usepackage[colorlinks,
            linkcolor=red,
            anchorcolor=blue,
            citecolor=blue
            ]{hyperref}
\usepackage{color, colortbl}
\usepackage{xcolor}
\usepackage[dvipsnames]{xcolor}
\usepackage{fullpage}
\usepackage{hyperref}
\usepackage[protrusion=true,expansion=true]{microtype}
\usepackage{pbox}
\usepackage{setspace}
\usepackage{tabularx}
\usepackage{float}
\usepackage{wrapfig,lipsum}
\usepackage{enumitem}
\usepackage{microtype}
\usepackage{graphicx}
\usepackage{enumerate}
\usepackage{enumitem}
\usepackage{pgfplots}
\usepackage[T1]{fontenc}
\usepackage{inconsolata} 
\usetikzlibrary{arrows,shapes,snakes,automata,backgrounds,petri}
\usepackage{booktabs} 

\newcommand{\problemdivider}[1][gray!45]{%
  \par\medskip
  \noindent{\color{#1}\rule{\linewidth}{0.6pt}}%
  \par\medskip
}

\newcommand{\dist}{\cT}

\newcommand{\ini}{\text{ini}}
\newcommand{\refe}{\text{ref}}
\newcommand{\llm}{\text{LLM}}

\usepackage[most]{tcolorbox}

\usepackage{enumitem}
\usepackage{caption}
\definecolor{lightblue}{HTML}{a8dadc}
\definecolor{green}{HTML}{76c893}
\definecolor{coralred}{HTML}{f28482}
\definecolor{purple}{HTML}{7b8cde}
\newcommand{\piref}{\pi_{\text{ref}}}
\usepackage{subcaption}

\makeatletter
\newcommand*{\rom}[1]{\expandafter\@slowromancap\romannumeral #1@}
\makeatother
\title{\huge On the Limits of Test-Time Compute: Sequential Reward Filtering for Better Inference}

\author
{
	Yue Yu\thanks{
Department of Statistics, Indiana University Bloomington, IN 47405, USA; e-mail: {\tt yyu3@iu.edu}} 
	~~~and~~~
	Qiwei Di\thanks{Department of Computer Science, University of California, Los Angeles, 
Los Angeles, CA 90095, USA; e-mail: {\tt qiwei2000@cs.ucla.edu}} 
	~~~and~~~
	Quanquan Gu\thanks{Department of Computer Science, University of California, Los Angeles, 
Los Angeles, CA 90095, USA; e-mail: {\tt qgu@cs.ucla.edu}} 
	~~~and~~~
	Dongruo Zhou\thanks{Department of Computer Science, Indiana University Bloomington, IN 47408, USA; e-mail: {\tt dz13@iu.edu}} 
}

\begin{document}
\date{}
\maketitle

\begin{abstract}
Test-time compute (TTC) has become an increasingly prominent paradigm for enhancing large language models (LLMs). Despite the empirical success of methods such as best-of-$n$ (BoN) sampling and sequential revision, their fundamental limits remain unclear. We address this gap by analyzing a mixture-of-reference policy model and proving that standard BoN is inherently suboptimal. To move closer to the optimal frontier, we study reward-filtered sequential inference, a simple procedure that selectively incorporates only high-reward generations into the context. This mechanism concentrates computation on superior policy candidates and suppresses inferior ones. On the theoretical side, we show that reward-filtered sequential inference yields strictly stronger guarantees than standard TTC paradigms. On the empirical side, we evaluate such an inference strategy across diverse benchmarks and observe consistent improvements over widely used approaches, demonstrating the practical effectiveness of our framework.
\end{abstract}

\section{Introduction}

How to effectively utilize large language models (LLMs) for solving new tasks has become a central research question. Among the many approaches, Test-Time Compute (TTC) has recently attracted significant attention. The key idea of TTC is to allocate additional computation during inference to improve task performance. Unlike post-training approaches such as fine-tuning or reinforcement learning, TTC requires no additional training of the base model. As a result, inference-time alignment methods provide a lightweight yet powerful alternative that greatly simplifies deployment. Well-known TTC methods include Best-of-N (BoN) sampling, chain-of-thought (CoT) reasoning, and their many variants \citep{stiennon2020learning, nakano2021webgpt, wei2022chain, wang2022self, zhou2022least, yao2023react, yao2023tree, chen2022program, shinn2023reflexion}.

In this work, we specifically focus on a family of TTC methods that leverage an external reward model, which can assign a score to each generated answer. Perhaps the most basic and widely studied approach in this family is BoN, a representative strategy of \emph{parallel TTC}. Given a problem $x$, the LLM generates $n$ answers $a_1,\dots,a_n$, and then selects the one with the highest reward score. BoN is highly intuitive, easy to implement, and has been empirically shown to substantially outperform vanilla decoding across a wide range of tasks. Naturally, this raises a fundamental question:

\begin{center}
\textit{Is Best-of-N the best one can hope for under a fixed test-time budget?}
\end{center}

Recent work \citep{snell2024scaling} suggests that the answer may be negative. Unlike BoN, which allocates its test-time budget in parallel by generating multiple independent samples, another line of research explores \emph{sequential TTC}. In sequential TTC, generated answers are inserted back into the input window of the LLM, effectively reshaping its conditional distribution as more computation is spent. Empirically, this approach has demonstrated performance gains beyond BoN by gradually biasing the model towards higher-quality outputs. However, despite its practical promise, the theoretical understanding of sequential TTC remains very limited. This motivates our second central question:

\begin{center}
\textit{What is the best achievable strategy under sequential TTC?}
\end{center}

In this work, we formalize TTC with reward models as a decision problem over this mixture-of-reference-policies model and use it to address both questions above. Conceptually, our goals are two-fold: (i) characterize the best-achievable regret under a fixed test-time budget for parallel TTC algorithms such as BoN, and (ii) design a sequential TTC procedure that, using the same reward model, provably attains a strictly better budget–performance trade-off. Our key contributions are summarized as follows:

\begin{itemize}[leftmargin = *]
    \item \textbf{Limits of parallel TTC.} We study the fundamental limit of TTC under a \emph{mixture-of-reference policies} model, where the pretraining data of the LLM is assumed to consist of trajectories generated by multiple underlying policies. Within this setting, we establish a lower bound on the test-time budget required to achieve a near-optimal policy and show that parallel TTC methods, including BoN, fall short of this bound, indicating that despite its popularity, parallel TTC is not optimal under more realistic modeling assumptions.
    \item \textbf{A reward-filtered sequential TTC method.} To go beyond parallel TTC, we study a simple sequential TTC method called \emph{Reward-Filtered Sequential Best-of-n} (RF-SeqBoN), where only high-reward generations are fed back into the LLM’s input window. This procedure progressively refines the generation distribution towards that of the optimal policy induced by the reward function. We prove that, under mild assumptions on the reward model, our method strictly improves upon parallel TTC and achieves significantly larger gains on harder tasks.
    \item \textbf{Experiments.} We conduct extensive experiments across diverse benchmarks and backbone LLMs. Empirical results consistently demonstrate that our approach achieves higher \emph{test-time budget efficiency} than existing baselines, confirming that the theoretical advantages of our method translate into substantial practical improvements. 
\end{itemize}

In the rest of our paper, we formalize the sequential sample-and-evaluate framework for TTC with reward models and specify our mixture-of-reference-policies pretraining assumption in Section~\ref{sec:preliminaries}. Section~\ref{sec:general} then develops fundamental lower bounds on the test-time budget required by parallel TTC schemes, showing that vanilla BoN is statistically suboptimal under this model. Building on these insights, Section~\ref{sec:newalg} introduces our \emph{Reward-Filtered Sequential Best-of-$n$ (RF-SeqBoN)} algorithms, together with assumptions on the reward model, and establishes improved sample-complexity guarantees that strictly dominate parallel TTC in appropriate regimes. Section~\ref{sec:experiments} presents empirical evaluations across multiple benchmarks and backbone LLMs, including ablations on key design choices, demonstrating consistent gains in budget efficiency. Finally, Section~\ref{sec:conclusion} concludes with a summary of our theoretical and empirical work.

\paragraph{Notation.}
We use lowercase letters to denote scalars, and bold lowercase (resp. uppercase) letters to denote vectors (resp. matrices). For an integer $n \in \mathbb{N}$, let $[n] := \{1,2,\dots,n\}$. 
For two nonnegative functions $a(x)$ and $b(x)$ defined on the same domain, we write $a(x) \lesssim b(x)$ if there exists an absolute constant $C > 0$ such that $a(x) \leq C\, b(x)$ for all $x$. 
Let $\cV^*$ denote the set of all finite token sequences. 
We define two finite subsets: the initial prompt space $\cX \subseteq \cV^*$ and the action space $\cA \subseteq \cV^*$. 
Each action $a \in \cA$ corresponds to a complete response represented as an autoregressively generated token sequence.
A policy conditions on a given prompt or, more generally, on any sequence $h \in \cV^*$, and is written as $\pi(\cdot \mid h)$.
We use $\Pi$ to denote the class of candidate policies.

\section{Related Work}

\paragraph{Parallel TTC.}
Several lines of work have explored parallel TTC. Broadly, two strategies have emerged. The first is based on the self-consistency approach, where multiple answers are generated in parallel from one or more LLMs and the final output is chosen by majority vote \citep{wang2022self,brown2024large,chen2024more}. This method is simple and easy to implement but relies solely on the intrinsic ability of the LLM and often falls short of achieving the best performance. The second line of work augments parallel TTC with an external reward model, selecting the final answer according to reward scores \citep{song2024good,irvine2023rewarding,puri2025probabilistic}. Additional parallel TTC methods introduce fine-tuning into the pipeline \citep{sessa2024bond,chow2024inference}. \citet{zuo2025strategic} consider allocate \emph{across}-query allocation across different questions to further improve efficiency under fixed budgets, orthogonal to our \emph{within}-query controller. Our work is most closely related to the second line. In contrast, we focus on improving the rate of inference through history-conditioned gating under rewards. In our experiments, we carefully control for identical verifiers and token budgets so that improvements cannot be attributed solely to stronger re-rankers, and they have been verified through multiple budget levels \citep{wu2024inference}.

\paragraph{Sequential TTC.}
Unlike parallel TTC, sequential TTC explicitly decides \emph{when} to spend extra steps and \emph{where} to revise. Classical CoT techniques \citep{chen2023teaching, zhang2024small, lee2025evolving, yao2023react} have been extensively studied without relying on external reward models. With reward models, several sequential TTC approaches have been explored. For example, \citet{munkhbat2025self} propose a few-shot BoN method that leverages a powerful external LLM to generate demonstrations and then selects answers via BoN. Iterative self-refinement and policy-as-verifier approaches provide alternative architectures that can be used as verifiers to trigger revisions \citep{madaan2023self,jiang2025pag}. Other directions include uncertainty-aware step-wise verification \citep{ye2025uncertainty}, universal/self-consistency methods for open-ended outputs \citep{chen2023universal,kang2025scalable}, and “PRMs that think” \citep{khalifa2025process}, all of which further enhance the judge’s signal. By contrast, our setup belongs to this line of work but deliberately adopts a minimalist design, avoiding complex techniques such as Tree- or Graph-of-Thoughts \citep{yao2023tree,besta2024graph}, while still demonstrating improved performance through sequential TTC.

\paragraph{Theory of TTC.}
A growing theory literature clarifies when extra samples help, when they hurt, and how to apportion TTC. For BoN, \citet{beirami2024theoretical} correct the folklore $\mathrm{KL}$ identity and bound win-rate improvements, while \citet{huang2025best} establish a coverage–error frontier that reveals BoN’s reward hacking at large $N$ and propose $\chi^2$-regularized sampling with skyline-optimal, scaling-monotone guarantees. Our mechanism achieves a comparable regularization effect by \emph{concentrating the proposal} through reward-filtered histories rather than reweighting selection. On sample complexity, \citet{huang2025sample} separate self-consistency from BoN. \citet{foster2025good} study how the base model performs in TTC in terms of coverage and the benefits of multi-turn exploration. Recently, \citet{xu2025provably} analyze Learning from language/process feedback (\textsc{HELiX}), provide regret guarantees under latent rewards, and highlight the need for richer, process-aware signals. By contrast, our work focuses on establishing a separation result between parallel and sequential TTC, and on developing an algorithm that uses reward signals to guide sequential TTC—a direction largely absent from prior work.

\section{Preliminaries}
\label{sec:preliminaries}

\paragraph{Task Description.}
We assume access to a large language model (LLM) $\pi_{\llm}(\cdot \mid h)$, which receives an input sequence $h \in \cV^*$, and then outputs a distribution over subsequent sequences. In our setting, we restrict the output sequences to lie in $\cA$. Let $\cX$ be the prompt set and $p_{\ini} \in \Delta(\cX)$ denote a distribution over initial prompts. In each round, we first draw $x \sim p_{\ini}$. Then an algorithm $\text{Alg}$ interacts with the language model $\pi_{\llm}$ for multiple times before producing a final action $\hat a \in \cA$. We denote the resulting conditional distribution of $\hat a$ by $\pi_{\text{Alg}}(\cdot \mid x)$. We assume a reward function $r: \cA \times \cX \to [-1,1]$, about which we make the following assumption:
\begin{assumption}\label{ass:reward}
    Given any $x \in \cX$, $a^\star(x):=\argmax_{a \in \cA} r(a,x)$ is unique. Meanwhile, $r(a^\star(x), x) = 1$. 
\end{assumption}
Because the output space $\cA$ can be extremely large, it is infeasible to explicitly evaluate either the language model $\pi_{\llm}$ or the reward function $r$ over all possible candidates. Instead, we assume that our algorithm operates under the following sequential sample-and-evaluate framework:

\begin{definition}[\emph{Sequential} sample-and-evaluate framework, generalized from \citealt{huang2025best}]
\label{def:sae}

For a given prompt $x \in \cX$, the algorithm may sequentially issue $n$ queries $\{h_i\}_{i=1}^n$. For each $i$, it samples $a_i \sim \pi_\llm(\cdot \mid h_i)$, 

and observes the reward value $r(a_i, x)$ . The efficiency (query complexity) of the algorithm is measured by the total number of queries $n$.
\end{definition}
\begin{remark}
    \citet{huang2025best} studied a more restricted setting of the sample-and-evaluate framework, 
    where one can only sample actions from a fixed prompt $x$, i.e., from $\pi_{\llm}(\cdot \mid x)$. 
    While conceptually simple, this framework largely ignores the internal structure and sequential 
    nature of LLMs.
\end{remark}

We evaluate performance with respect to a \emph{comparator policy} $\pi^\star (\cdot| x): \cX \rightarrow \Delta(\cA)$. Given a prompt $x$ and comparator policy $\pi^\star$, we define the regret of an action $\hat a$ as
\begin{align}
\text{Regret}(\hat a; x, \pi^\star)
:= \EE_{a \sim \pi^\star(\cdot| x)} \big[ r(a,x) - r(\hat a,x) \big].\notag
\end{align}
An action $\hat a$ is $\epsilon$-optimal with respect to $\pi^*$ if $\text{Regret}(\hat a; x, \pi^\star) \leq \epsilon$ for some $\epsilon \in [0,1]$.  
When $\hat a$ is generated by an algorithm $\text{Alg}$, we write $\text{Regret}(\text{Alg}; x, \pi^\star)
:= \EE_{\hat a \sim \pi_{\text{Alg}}(\cdot \mid x)}
   \text{Regret}(\hat a; x, \pi^\star)$. 
Our goal is to design algorithms that, for prompts drawn from $p_{\ini}$, achieve small regret with respect to $\pi^\star$ while minimizing the \emph{sample complexity}, defined as the total number of times that the algorithm queries $\pi_{\llm}$.

\paragraph{Pretraining of LLM.}
Next, we impose a structural assumption on the pretraining of $\pi_\llm$. Specifically, we assume access to a pretraining dataset $\cD$, consisting of trajectories collected from a reference policy $\piref$. Formally, there exists 
a finite family of reference policies
$\{ \pi_{\refe}^\tau(\cdot \mid x) : \tau \in \mathcal{T}_{\refe} \}$,  
where $\mathcal{T}_{\refe}$ is a finite index set and  
$p_{\refe}$ denotes the prior distribution reflecting the proportion of data contributed by each reference policy.
In practice, different reference policies may correspond to distinct answer-generation styles in the pretraining corpus. For instance, some may produce concise responses while others generate more elaborate or vivid explanations. These stylistic variations naturally form a finite collection of reference policies contributing to the dataset.
The trajectories $h^t = (x^t, a_1^t, \dots, a_N^t)$ in the pretraining dataset $\cD$ are generated as follows:
\begin{itemize}[leftmargin=*, nosep]
    \item Sample an initial prompt $x^t \sim p_{\ini}$ and a reference index $\tau^t \sim p_{\refe}$. Initialize the history $h_0^t = x^t$.
    \item For $i = 1, \dots, N$, sequentially draw an action $a_i^t \sim \pi_{\refe}^{\tau^t}(\cdot \mid x^t)$
    and update the history as $h_i^t = h_{i-1}^t \cup (a_i^t)$.  
    \item Set $h^t = h_N^t$ and add it to $\cD$.
\end{itemize}

\begin{remark}
In our assumption for pretraining data, all actions within the same trajectory are sampled from the same distribution $\pi_{\refe}^\tau(\cdot \mid x)$. Our assumption on the pretraining distribution is aligned with those commonly made for in-context learning \citep{xie2021explanation,zhang2023and}. In particular, our pretraining dataset does not contain any reward information. This contrasts with the in-context reinforcement learning (ICRL) literature \citep{lin2024transformers,wang2024transformers,lee2023supervised}, which typically assumes that rewards are included in the pretraining data.
\end{remark}

We then interpret the pretrained LLM $\pi_{\llm}$ as being trained on $\cD$ and we call $\pi_\llm$ as a \emph{mixture of reference policy model} since it consists of trajectories drawn from different reference policies. Conceptually, when the pretraining length $T$ and the number of data $N$ grow to infinity, $\pi_\llm$ converges to the conditional distribution induced by the data-generating process \citep{lee2023supervised}. We formalize this as the following assumption.

\begin{assumption}[Realizability of $\pi_{\llm}$]\label{ass:ideal}
For any trajectory $h \in \cV^*$, we have $\pi_{\llm}(\cdot \mid h) = \PP_{\cD}(\cdot \mid h)$, 
where $\PP_{\cD}(\cdot \mid h)$ denotes the true conditional distribution of the next action 
given history $h$ under the data-generating process defined by sampling $\tau \sim p_{\refe}$, 
$x \sim p_{\ini}$, and actions from $\pi_{\refe}^\tau(\cdot \mid x)$.  
\end{assumption}

\section{Fundamental limits of parallel test-time compute}\label{sec:general}

\begin{algorithm}[t!]
\caption{Sequential Best-of-$n$ (SeqBoN)}\label{alg:1}
\begin{algorithmic}[1]
\REQUIRE Prompt $x$, reward $r(\cdot,\cdot)$, budget $n$. 
\FOR{$i = 1, \dots, n$}
\STATE Update $h_i$ based on $x,a_1,\dots, a_{i-1}$. Sample $a_i\sim \pi_\llm(\cdot \mid h_i)$.
\ENDFOR
\ENSURE Return $\hat a = \argmax_{i \in [n]} r(a_i, x)$
\end{algorithmic}
\end{algorithm}

In this section, we demonstrate that under a more refined assumption on the LLM-induced distribution $\pi_\llm$, the vanilla Best-of-N (BoN) algorithm is suboptimal, indicating the need for a more nuanced algorithmic design. For this purpose, we first consider the sample complexity bound for the vanilla BoN algorithm, which has been studied in \cite{huang2025best}. We will study this problem in a more general framework. To begin with, we first propose a Sequential BoN (SeqBoN) in Algorithm \ref{alg:1}, which serves as a meta-algorithm. At each iteration $i$, Algorithm \ref{alg:1} will formulate an input sequence $h_i$ based on the prompt $x$ and the previous answers $a_1,\dots, a_{i-1}$, then samples a new answer $a_i$ from $\pi_\llm(\cdot \mid h_i)$. Apparently, Algorithm \ref{alg:1} takes the classical BoN algorithm as its special case when $h_i = x$ for all $i$. We first introduce $\cE_M$-divergence and coverage. 
\begin{definition}[$\cE_M$-divergence and coverage, \citealt{huang2025best}]
\label{def:cEM}
Let $\pi_1, \pi_2 \in \Pi$ be two policies over the action space $\cA$.  
For a rejection threshold $M \ge 1$, the $\cE_M$-divergence between $\pi_1$ and $\pi_2$ is defined as $\cE_M(\pi_1, \pi_2) := \sum_{a \in \cA} \max\{0, \pi_1(a) - M \pi_2(a)\}$.  
We denote by $M^\epsilon_{\pi_1, \pi_2}$ the smallest $M$ such that $\cE_M(\pi_1, \pi_2) \leq \epsilon$.  
We also define the \emph{coverage} as $C(\pi_1, \pi_2) := \EE_{a \sim \pi_1}\bigl[\pi_1(a)/\pi_2(a)\bigr]$. Moreover, for any $0 < \epsilon < 1$, if $C(\pi_1, \pi_2)<\infty$, then the rejection threshold satisfies the upper bound $M^\epsilon_{\pi_1, \pi_2} \leq C(\pi_1, \pi_2)/\epsilon$.
\end{definition}

First, based on Definition \ref{def:cEM}, we recall the sample complexity result of vanilla BoN built in \citet{huang2025best}, which is near-optimal under the parallel TTC setup.

\begin{proposition}[Adapted from \citealt{huang2025best}]
\label{prop:bon}
    Given $\epsilon>0$, denote $M^{x, \epsilon}_\llm: = M^{\epsilon}_{\pi^\star(\cdot \mid x), \pi_\llm(\cdot\mid x)}$ and $C_\llm^\star(x):= C(\pi^\star(\cdot \mid x), \pi_{\llm}(\cdot \mid x))$. Then vanilla BoN (Algorithm \ref{alg:1} with $h_i = x$) takes $n = M^{x, \epsilon}_\llm\cdot \log(1/\epsilon) = O(C_\llm^\star(x)/\epsilon)$ samples to achieve $\text{Regret}(\text{BoN}; x, \pi^\star) \leq 2\epsilon$. Meanwhile, for any parallel TTC algorithm, there exists a problem instance such that it can not return an $\epsilon$-optimal answer if $n<M^{x, \epsilon}_\llm$. 
\end{proposition}

Although Proposition \ref{prop:bon} suggests that the optimal sample complexity for parallel TTC should be proportional to $M^{x, \epsilon}_\llm$, our next theorem establishes a lower bound on the sample complexity under the sequential sample-and-evaluate framework, thereby separating the parallel and sequential settings.

\begin{theorem}[Lower Bound of Sequential Sample-and Evaluate Algorithms]\label{thm:lower}
    Suppose the comparator policy $\pi^\star(\cdot |x):= a^\star(x)$ is a deterministic policy. Let $\tau^\star(x) \in \dist_\refe$ be the reference policy index satisfying $\tau^\star(x) \in \arg\max_{\tau\in \dist_{\refe}} \log \pi_{\refe}^\tau(a^\star(x) \mid x)$. 
    Meanwhile, let $M^{x, \epsilon}_{\tau^\star(x)}:=M^\epsilon_{\pi^\star(\cdot| x), \pi_\refe^{\tau^\star(x)}(\cdot \mid x)}$. Then 
\begin{itemize}[leftmargin=*]
    \item For any sequential sample-and-evaluate algorithm $A$, there exists a reward function $r$ such that if $n < M^{x,\epsilon}_{\tau^\star(x)}$, we must have $\text{Regret}(A; x, \pi^\star) > \epsilon$.
    \item For any $\epsilon$, we have $M^{x,\epsilon}_{\tau^\star(x)} \leq M^{x,\epsilon}_{\llm}$, with strict inequality whenever the reference policies $\pi_{\refe}^\tau(\cdot \mid x)$ are not identical across $\tau$.
\end{itemize}
\end{theorem}

\begin{proof}[Proof Sketch]
First, recall the definition of $\pi_\llm$, which is the posterior distribution over $\cD$. Then we can further write 
\begin{align}
\label{eq:posterior}
\pi_{\llm}(\cdot \mid h) = \PP_{\cD}(\cdot \mid h) = \sum_{\tau\in \dist_{\refe}} \PP_{\cD}(\cdot\mid h, \tau) \PP_{\cD}(\tau \mid h) = \sum_{\tau\in \dist_{\refe}} \pi_{\refe}^\tau(\cdot \mid x)\PP_{\cD}(\tau \mid h).
\end{align}
Therefore, for any input sequence $h$, $\pi_{\llm}(\cdot \mid h)$ can be treated as a Bayesian aggregation \citep{xie2021explanation, zhang2023and, hoeting1999bayesian, jacobs1991adaptive, jordan1994hierarchical} of the reference policies $\{ \pi_{\refe}^\tau(\cdot \mid x) : \tau \in \dist_\refe \}$. Then since $\pi_\refe^{\tau^\star(x)}(\cdot \mid x)$ has the maximum probability to sample $a^\star(x)$, then we know that $\pi_{\llm}(a^\star(x) \mid h) \leq \pi_\refe^{\tau^\star(x)}(a^\star(x) \mid x)$ for any input sequence $h$, which leads to our final bound. 
\end{proof}

Overall, our lower bound in Theorem \ref{thm:lower} indicates that, under the sequential-type pretraining assumption of LLMs, the statistical limit achievable by any sequential-type algorithm can be strictly better than that of purely parallel methods such as vanilla BoN. This motivates the development of more efficient algorithms that explicitly leverage this revised statistical landscape.

\section{Better Test-Time Compute with Reward-Filtered Sequences}\label{sec:newalg}

In this section, we propose several algorithms designed to achieve this new statistical limit. 
We begin with an analysis of Algorithm~\ref{alg:1}. 
The following theorem provides a regret bound between the comparator policy $\pi^\star$ and SeqBoN.

\begin{theorem}\label{thm:general_bon}
For any prompt $x$ and any $0 < \epsilon < 1$, 
\begin{align}
\text{Regret}(\text{SeqBoN}; x, \pi^\star) 
\leq \epsilon 
+ \EE_{a_1,\dots,a_n} \Big[ \exp\Big(- n^2 \Big/\Big(\sum_{i=1}^n M^\epsilon_{\pi^\star(\cdot \mid x),\, \pi_\llm(\cdot \mid h_i)}\Big) \Big) \Big]. \notag
\end{align}
\end{theorem}

\begin{remark}
\citet{huang2025best} showed that the vanilla BoN method satisfies $\text{Regret}(\text{BoN}; x, \pi^\star) 
\leq \epsilon 
+ \exp(- n/{M_\llm^{x, \epsilon}} )$. Our bound is tighter whenever $M^\epsilon_{\pi^\star(\cdot \mid x),\, \pi_\llm(\cdot \mid h_i)} 
\leq M^\epsilon_{\pi^\star(\cdot \mid x),\, \pi_\llm(\cdot \mid x)}$, 
which indicates that conditioning on history $h_i$ allows $\pi_\llm(\cdot \mid h_i)$ 
to provide a more accurate approximation of $\pi^\star(\cdot \mid x)$ than $\pi_\llm(\cdot \mid x)$. 
\end{remark}

\begin{figure}[t!]
\centering
\begin{minipage}{0.48\linewidth}
\begin{algorithm}[H]
\caption{RF-SeqBoN with burn-in}
\label{alg:2.5}
\begin{algorithmic}[1]
\REQUIRE $\pi_{\llm}, x, r(\cdot,\cdot), \gamma, m$
\STATE Initialize $\bar h \gets \langle x\rangle$
\FOR{$i = 1, \dots, n$}
  \STATE {\color{blue}\textbf{If} $|\bar h| \geq m$ \textbf{then} $h_{i} \gets \bar h$ \textbf{else} $h_i \gets \langle x \rangle$}
  \STATE Sample $a_i \sim \pi_{\llm}(\cdot \mid h_i)$
  \STATE \textbf{If} $r(a_i, x) \ge \gamma$ \textbf{then} $\bar h \gets \bar h \,\|\, a_i$ 
\ENDFOR
\ENSURE $\hat a = \argmax_{i \in [n]} r(a_i, x)$
\end{algorithmic}
\end{algorithm}
\end{minipage}\hfill
\begin{minipage}{0.48\linewidth}
\begin{algorithm}[H]
\caption{RF-SeqBoN}
\label{alg:2}
\begin{algorithmic}[1]
\REQUIRE $\pi_{\llm}, x, r(\cdot,\cdot), \gamma$
\STATE Initialize $\bar h \gets \langle x\rangle$
\FOR{$i = 1, \dots, n$}
  \STATE {\color{blue}$h_i \gets \bar h$}
  \STATE Sample $a_i \sim \pi_{\llm}(\cdot \mid h_i)$
  \STATE \textbf{If} $r(a_i, x) \ge \gamma$ \textbf{then} $\bar h \gets \bar h \,\|\, a_i$
\ENDFOR
\ENSURE $\hat a = \argmax_{i \in [n]} r(a_i, x)$
\end{algorithmic}
\end{algorithm}
\end{minipage}
\end{figure}

Recall that, to achieve the lower bound $M^{x,\epsilon}_{\tau^\star(x)}$, 
it suffices to design an algorithm that constructs histories $h_i$ such that 
$\pi_\llm(\cdot \mid h_i) \to \pi_{\refe}^{\tau^\star(x)}(\cdot \mid x)$. 
This can be accomplished if $h_i$ consists of actions drawn from the optimal reference policy 
$\pi_{\refe}^{\tau^\star(x)}(\cdot \mid x)$. 
Such a strategy has been widely used in in-context learning, 
where a number of works establish theoretical guarantees on how many context samples are needed 
to achieve this convergence \citep{wies2023learnability, zhang2023and, li2023transformers, bai2023transformers}. 
However, in our setting, the reference policies $\pi_{\refe}$ are inaccessible at evaluation time, 
since they are only available during pretraining.

We make the following observation: although we cannot directly sample from the optimal reference policy 
$\pi_{\refe}^{\tau^\star(x)}(\cdot \mid x)$, we can instead construct $h_i$ as a \emph{sequence of optimal actions} 
$a^\star(x)$, which still ensures that $\pi_\llm(\cdot \mid h_i) \to \pi_{\refe}^{\tau^\star(x)}(\cdot \mid x)$. 
To see this, consider the history $h = (\underbrace{a^\star(x), \dots, a^\star(x)}_{k \text{ times}}, x)$. Given the fact that $\tau^\star(x)$ maximizes $\pi_{\refe}^\tau(a^\star(x) \mid x)$ and  
$\PP_{\cD}(\tau^\star(x) \mid h) \to 1$ as $k \to \infty$, under our Assumption \ref{ass:ideal}, we have
\begin{small}
\begin{align}
    \pi_{\llm}(\cdot \mid h) 
    &= \sum_{\tau \in \dist_{\refe}} \pi_{\refe}^\tau(\cdot \mid x) \, \PP_{\cD}(\tau \mid h) = \sum_{\tau \in \dist_{\refe}} \pi_{\refe}^\tau(\cdot \mid x) \,
    \frac{\pi_\refe^\tau(a^\star(x) \mid x)^k \, p_{\refe}(\tau)}
         {\sum_{\tau' \in \dist_{\refe}} \pi_\refe^{\tau'}(a^\star(x) \mid x)^k \, p_{\refe}(\tau')}\rightarrow \pi_{\refe}^{\tau^\star(x)}(\cdot \mid x). \notag
\end{align}
\end{small}

Based on this observation, we investigate a simple variant of Algorithm~\ref{alg:1}, termed the \emph{Reward-Filtered Sequential Best-of-$n$ (RF-SeqBoN)} algorithm, with two instantiations presented in Algorithms~\ref{alg:2.5} and~\ref{alg:2}. At a high level, RF-SeqBoN can be viewed as a hybrid of BoN and the rewind-and-repeat strategy \citep{kim2024guaranteed,beirami2024theoretical}. At each iteration, the algorithm refines the answer-generation distribution by appending to the input window only those answers whose reward exceeds a threshold $\gamma$, thereby biasing the history toward near-optimal solutions. A larger $\gamma$ enforces stricter filtering at the expense of efficiency, whereas in practice a good balance is often achieved with $\gamma < 1$. Furthermore, Algorithm~\ref{alg:2.5} introduces a \emph{burn-in} parameter $m$ to control the length of the input sequence $h_i$.

\begin{remark}
Our studied algorithm falls under the general selection inference (SI) framework \citep{creswell2022selection,hu2024unveiling}, where a selection module is employed to identify in-context examples. However, vanilla SI typically relies on the LLM itself, often using likelihood scores to select the most probable examples from the history. This strategy overlooks task-specific information encoded by the reward signal. In contrast, Algorithms~\ref{alg:2.5} and~\ref{alg:2} select histories based on reward scores, thereby better capturing task-relevant information.
\end{remark}

In order to make sure that our approximation of $a^\star(x)$ is good enough, we have the following assumption on the reward model. 
\begin{assumption}\label{assumption:maxref}
For any $x$, there exist a threshold $\gamma^\star(x) < 1$, a margin $\Delta(x) > 0$, and ${\tau^\star(x)\in \dist_{\refe}}$ such that for any $a \in \cA$ satisfying $r(a,x) \geq \gamma^\star(x)$, we have 
\[
      \log \pi_{\refe}^{\tau^\star(x)}(a \mid x) - \sup_{\tau \neq \tau^\star(x)} \log \pi_{\refe}^\tau(a \mid x) \geq \Delta(x).
\]
Specifically, we denote $p_{\llm}(x): = \PP_{a \sim \pi_{\llm}(\cdot\mid x)} (r(a,x)\geq \gamma^\star(x))$. 
\end{assumption}
Assumption~\ref{assumption:maxref} requires that any action whose reward exceeds the threshold $\gamma^\star(x)$ can be reliably attributed to the optimal reference policy $\tau^\star(x)$, in the sense that its generation probability under $\tau^\star(x)$ is at least $\Delta$ larger than under any other reference policy. This margin condition ensures that high-reward actions are statistically distinguishable from those generated by suboptimal policies, enabling the algorithm to identify $\tau^\star(x)$ with high confidence from a finite number of samples. This assumption can be satisfied under mild conditions (please refer to Appendix~\ref{app:condition}). When it is violated, the reward feedback loop in RF-SeqBoN can in principle amplify mis-specification and lead to inference-time reward hacking; we discuss these risks and practical mitigation strategies in Appendix~\ref{app:reward-hacking}. With this assumption, we come to present our main results.

\begin{theorem}\label{thm:reward_filter}
    Under Assumptions \ref{ass:reward} , \ref{ass:ideal} and \ref{assumption:maxref}, suppose $\epsilon < 1-\gamma^\star(x)$ and $\text{supp}(\pi^\star(\cdot \mid x)) \subseteq \{r(a,x) \geq \gamma^\star(x)\}$. Let $p_\llm(x)$ be defined as in Assumption \ref{assumption:maxref}. 
Then set 
\begin{align}
\gamma &= \gamma^\star(x),\qquad m = \frac{\log [\epsilon^{-1}p_{\refe}^{-1}(\tau^\star(x))]}{\Delta(x)}, \qquad 
\bar n := \frac{m}{p_{\llm}(x)}, \notag \\
n &= \text{poly}\log(1/\epsilon) \cdot 
\begin{cases}
    M^{x, \epsilon}_{\llm} 
    & ,\text{if } M^{x, \epsilon}_{\llm} \leq \bar n, \\[6pt]
    M^{x, \epsilon}_{\tau^\star(x)} + \sqrt{\bar n \big(M^{x, \epsilon}_{\llm} - M^{x, \epsilon}_{\tau^\star(x)}\big)} 
    & ,\text{if } M^{x, \epsilon}_{\llm} > \bar n.
\end{cases} \label{help:0}
\end{align}
 for $\pi^\star$, the output policy of Algorithm \ref{alg:2.5} is $4\epsilon$-optimal to $\pi^\star$. 
\end{theorem}

\begin{remark}
The sample complexity of Algorithm \ref{alg:2.5} has two different regimes, dependent on $M^{x, \epsilon}_{\llm}$, which represents the necessary samples to use vanilla BoN to solve $x$, and a threshold $\bar n$. When $M^{x, \epsilon}_{\llm} \leq \bar n$, our Algorithm \ref{alg:2.5} performs the same as vanilla BoN, which suggests sequential TTC only outperforms parallel BoN when the problem instance $x$ is hard to be resolved. 
When $M^{x, \epsilon}_{\llm}$ is large, the sample complexity of Algorithm \ref{alg:2.5} falls into another regime. It can be seen that the sample complexity can be shaved to $\sqrt{\bar n M^{x, \epsilon}_{\llm}}$ when the theoretical lower bound $M^{x, \epsilon}_{\tau^\star(x)} \ll M^{x, \epsilon}_{\llm}$, which suggests that our Algorithm \ref{alg:2.5} could save a lot sample complexity when there exists a `best' reference policy $\pi_\refe^{\tau^\star(x)}(\cdot \mid x)$ that finds the optimal actions efficiently. 
\end{remark}
\begin{remark}
Theorem \ref{thm:reward_filter} requires $\epsilon$ to be a small term upper bounded by $1-\gamma^\star(x)$ and the comparator policy $\pi^\star$ only distributes on the support set whose reward is large. Note that such requirements are actually mild: we are mostly interested in finding the near-optimal actions whose reward is large. We look forward to extending our result to general $\epsilon$ and $\pi^\star$ as future work.  
\end{remark}

Theorem \ref{thm:reward_filter} represents the general form of sample complexity. When we have the policy coverage assumption denoted in Definition \ref{def:sae}, we also have the following corollary that characterizes the sample complexity of Algorithm \ref{alg:2.5}.

\begin{corollary}\label{coro:1}
Following the assumption on $\pi^\star$ defined in Theorem~\ref{thm:reward_filter}, denote $C^\star_\llm(x)$ as in Proposition \ref{prop:bon} and 
denote $\kappa(x) := \max\big\{p_{\refe}(\tau^\star(x)),\, e^{-\Delta(x)} \big\}$. 
Then, for sufficiently small $\epsilon > 0$ satisfying
\begin{align}
\epsilon \leq \min\Bigg\{1-\gamma^\star(x), \;
\frac{C^\star_{\llm}(x)\, p_\llm(x)\,\Delta(x)}{\log p_\refe^{-1}(\tau^\star(x))}, \;
\frac{C^\star_{\llm}(x)\, p_\llm(x)\,\Delta(x)}{\log p_\refe^{-1}(\tau^\star(x))}\cdot \frac{\kappa^2(x)}{1 - \kappa(x)} \Bigg\}, \notag
\end{align}
Algorithm~\ref{alg:2.5} finds an $\epsilon$-optimal policy to $\pi^\star$ with sample complexity
\begin{align}
    n = \text{poly}\log(1/\epsilon)\cdot \kappa(x)\cdot \frac{C^\star_{\llm}(x)}{\epsilon}. \notag
\end{align}
\end{corollary}

For Algorithm \ref{alg:2} that removes the burn-in stage, we still have that asynmpotically, under the coverage assumption, Algorithm \ref{alg:2} achieves the same sample complexity to Algorithm \ref{alg:2.5}.

\begin{corollary}\label{coro:adaptive}
    Following the same setup in Corollary \ref{coro:1}, denote $d(x):=(1+ e^{-\Delta(x)}/p_{\refe}(\tau^\star(x)))$. Then, for sufficiently small $\epsilon > 0$ satisfying
\begin{align}
\epsilon \leq \min\Bigg\{1-\gamma^\star(x), 
\frac{C^\star_{\llm}(x)\PP_{a \sim \pi_{\refe}^{\tau^\star(x)}(\cdot \mid x)}(r(a,x) \geq \gamma^\star(x))\Delta(x)}{\log p_\refe^{-1}(\tau^\star(x))}\cdot \frac{\kappa(x)}{d^2(x)} \Bigg\}, \notag
\end{align}
Algorithm~\ref{alg:2.5} finds an $\epsilon$-optimal policy to $\pi^\star$ with the same sample complexity as Corollary~\ref{coro:1}. 
\end{corollary}

\begin{remark}
Corollaries~\ref{coro:1} and \ref{coro:adaptive} show that, compared with the parallel TTC with sample complexity  $\tfrac{C^\star_{\llm}(x)}{\epsilon}$, 
    Algorithm~\ref{alg:2.5} achieves a strictly lower sample complexity since $\kappa(x)<1$. Notably, the improvement gap will be large when the probability mass of the optimal reference policy $p_{\refe}(\tau^\star(x))$ is small, yet the gap $\Delta(x)$ between the optimal and suboptimal reference policies is large. 
\end{remark}
\begin{remark}
Corollaries~\ref{coro:1} and \ref{coro:adaptive} also reveal that $\epsilon$ must be sufficiently small. 
    The admissible range of $\epsilon$ is controlled jointly by the gap $\Delta(x)$ and the probability $p_\llm(x)$ that the LLM outputs actions with reward above the threshold. These findings are consistent with empirical observations: sequential TTC methods tend to outperform parallel TTC methods when the optimal trajectory is relatively rare but significantly more rewarding than alternatives \citep{snell2024scaling}.
\end{remark}

\section{Experiments}
\label{sec:experiments}

\paragraph{Models and Datasets.}
Following recent work on math-reasoning evaluation under test-time scaling~\citep{guha2025openthoughts,wang2025m1,agarwal2025first,otth2025maximizing}, we evaluate our approach on five benchmarks with verifiable answers, each targeting different aspects of reasoning. The first is \textbf{MATH500}, a 500-problem subset of the MATH dataset that preserves the competition-style short-answer format and topic distribution \citep{hendrycks2021measuring,lightman2023let}; performance is measured by exact-match accuracy on the final numeric or algebraic expression. The second is \textbf{GPQA–Diamond}, a 198-question high-agreement split of GPQA spanning biology, physics, and chemistry. This benchmark is designed to be “Google-proof,” with expert-authored and validated multiple-choice questions \citep{rein2024gpqa}; evaluation is based on standard multiple-choice accuracy.  We also consider competition-style math benchmarks: \textbf{AMC\textquoteright 23}, a 40-problem subset of the 2023 American Mathematics Competitions assessing non-routine problem-solving in algebra, geometry, number theory, and combinatorics \citep{maa2023amc}; \textbf{AIME\textquoteright 24}, consisting of the 30 problems from the 2024 AIME I/II, each requiring a three-digit answer in the range \([000,999]\) \citep{maa2024aime}; and \textbf{AIME\textquoteright 25}, analogous to AIME\textquoteright 24 but using the 2025 AIME I/II set \citep{maa2025aime}; all of the competition-style math benchmarks are evaluated using exact-match accuracy on the final integer answer. 

For foundation models, we use \textbf{Qwen3–4B–Instruct} and \textbf{Qwen3–0.6B–Thinking} in our primary experiments, recently released open-weight instruction-tuned models which exhibits competitive capabilities at small scales, including instruction following, multilingual usage, and reasoning, and are broadly available under permissive licenses \citep{yang2025qwen25technicalreport,yang2025qwen3}. Reward models (RMs) map candidate solutions to scalar signals. We use a Process Reward Model (PRM): \textbf{Llama3.1–8B–PRM–Deepseek–Data}, where we only use the aggregation reward for each answer we generate once, a scalar reward $r\in[0,1]$ \citep{rlhflow2024rlhf}. 

\paragraph{Baselines and Implementation.}
We compare \textbf{RF-SeqBoN} implemented following Algorithm \ref{alg:2} against several variants of Algorithm \ref{alg:1} to suggest the validity of the reward-filtering scheme. In detail, we compare Algorithm \ref{alg:2} with \textbf{vanilla BoN (BoN)}, which refers to Algorithm \ref{alg:1} with $h_i = \langle x \rangle$; \textbf{Pure sequential BoN (PureSeq)}: A single trajectory with a fixed revision schedule and no cross-trajectory selection, which refers to Algorithm \ref{alg:1} with $h_i = \langle x, a_1,\dots, a_{i-1} \rangle$. In practice, due to the maximum token limit of LLMs, for PureSeq and RF-SeqBoN we restrict the number of appended answers in $h_i$ to a fixed budget, maintaining the context as a sliding window of most recent generations in a first-in-first-out (FIFO) manner. Details are deferred to Appendix \ref{app:hyperparameters}. 

\begin{figure}[t]
    \centering
    \includegraphics[width=0.95\linewidth]{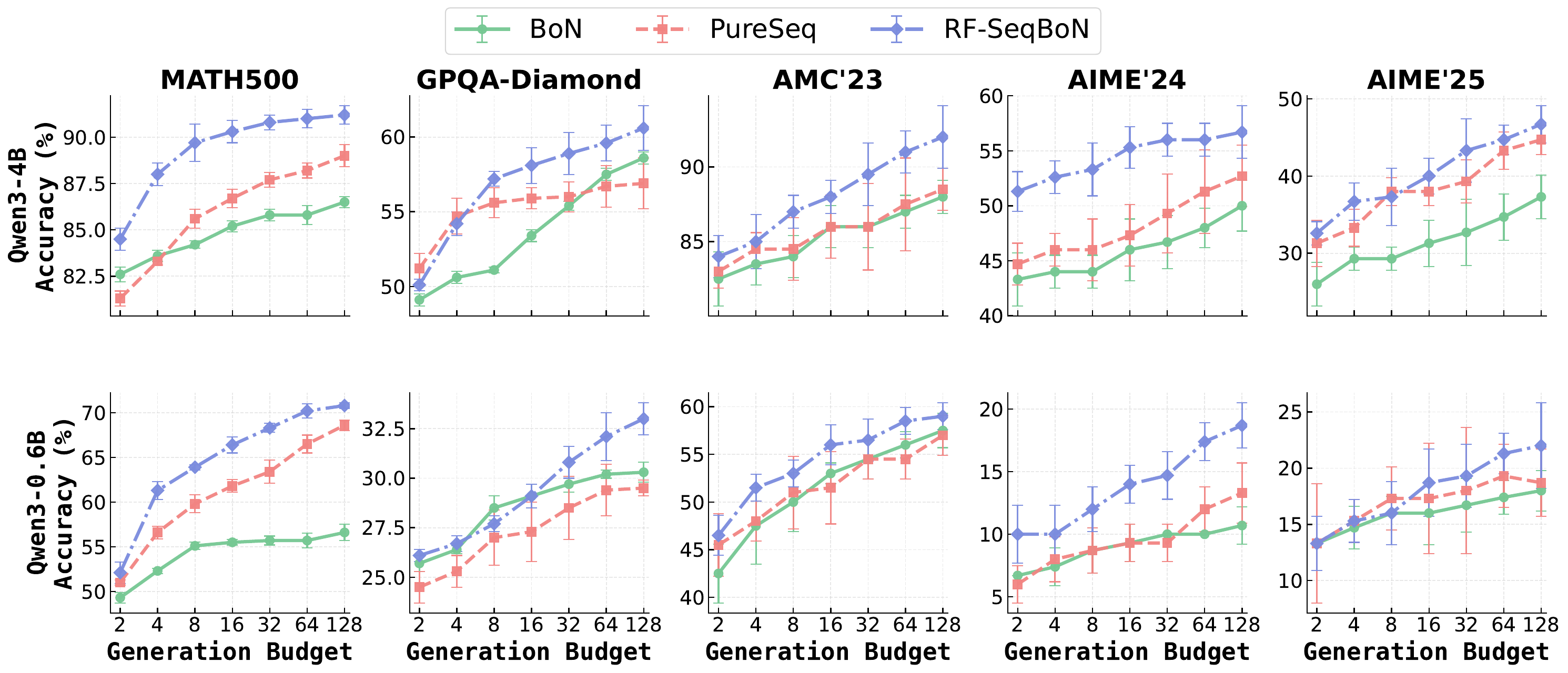}
    \caption{Main results with Qwen3-4B-Instruct \textbf{(top)} and Qwen3-0.6B-Thinking \textbf{(bottom)} foundation models of values with generation budget $N$. Each column corresponds to one benchmark dataset. The points and error bars show the mean and standard deviation across five repeated experiments, respectively.}
    \vskip -10pt
    \label{fig:main}
\end{figure}

\paragraph{Results. }
We characterize accuracy–compute trade-offs across datasets and models. Accuracy is measured via exact-match (for MATH500, AMC\textquoteright 23, AIME\textquoteright 24 and AIME\textquoteright 25) and multiple-choice accuracy (for GPQA–Diamond). Figure \ref{fig:main} summarizes accuracy as a function of total test-time budget \(N\) for BoN, PureSeq, and our method RF-SeqBoN across the three benchmarks. We have the following observations:
\begin{itemize}[leftmargin = *]
    \item RF-SeqBoN, which conditions the input on a history filtered by reward scores, consistently improves over vanilla BoN. This validates that maintaining a reward-biased history provides more informative context and leads to higher final accuracy. PureSeq also leverages history, but it keeps all past generations without filtering. As a result, it accumulates both useful and noisy answers, which makes it less stable than RF-SeqBoN on GPQA-Diamond and AIME'25 with larger budgets.  
    \item As \(N\) increases, we observe two regimes. On benchmarks like MATH500 and AIME'24, the advantage of RF-SeqBoN persists or even grows with larger budgets. In contrast, on GPQA-Diamond, AMC'23 and AIME'25, the baselines gradually catch up at high budgets, suggesting that all methods are approaching the intrinsic limit of the backbone model rather than continuing to benefit from more trajectories. 
    \item With the smaller model (Qwen3–0.6B-Thinking), absolute performance is lower, but RF-SeqBoN still consistently outperforms both BoN and PureSeq, suggesting that the reward-filtered history mechanism is robust across model scales.  
\end{itemize}

\begin{wrapfigure}{r}{0.5\textwidth}
    \centering
    \vspace{-7mm} 
    \includegraphics[width=0.48\textwidth]{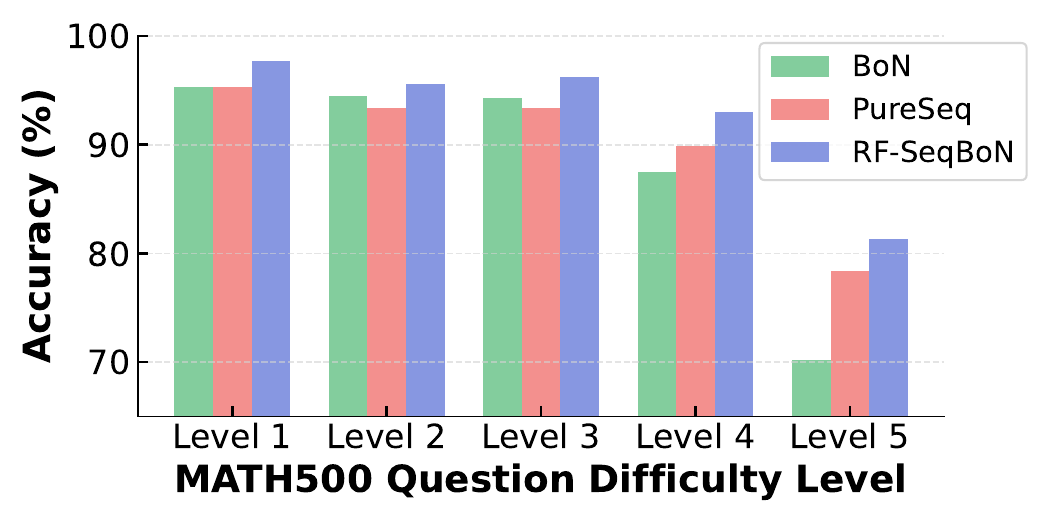}
    \caption{Breakdown of MATH500 performance when generation budget $N=128$ across five difficulty levels.}
    \label{fig:case_study_difficulty}
\end{wrapfigure}

\paragraph{Case Study. } We conduct a breakdown analysis of MATH500 by difficulty levels. We follow the level information brought by \citet{hendrycks2021measuring} ranging from 1 (easiest) to 5 (most challenging). 
The results of generation budget $N=128$ are presented in Figure~\ref{fig:case_study_difficulty}, using the Qwen3-4B-Instruct foundation model, consistent with the main results reported in Figure~\ref{fig:main}. We observe a clear trend of decreasing accuracy as the problem difficulty increases. 
In particular, PureSeq performs slightly worse than BoN on easier subsets (levels 1--3), but achieves substantial gains on more difficult subsets (levels 4--5). 
Across all difficulty levels, RF-SeqBoN consistently outperforms both baselines, demonstrating its robustness in handling problems of varying complexity.

\subsection{Ablation Studies}
\paragraph{Choice of Hyperparameter $\gamma$. }
\label{sec:ablation_gamma}
To quantify the effect of the reward–filter threshold $\gamma$, we evaluate RF-SeqBoN with Qwen3-4B-Instruct on MATH500 while varying $\gamma$. Figure~\ref{fig:ablation_combined}(a) reports accuracies across generation budgets. We observe that the choice of $\gamma$ indeed affects the performance of RF-SeqBoN. If $\gamma$ is too low, noisy trajectories persist and hinder self-refinement; if $\gamma$ is too high (e.g., $\gamma \approx 1$), most candidates are rejected and RF-SeqBoN degenerates toward BoN, forfeiting the benefits of using history for refinement. Meanwhile, our method is not particularly sensitive to $\gamma$ within a reasonable range: the performance for $\gamma = 0.93, 0.95, 0.97$ remains nearly identical and optimal, while only $\gamma = 0.9$ and $\gamma = 0.99$ exhibit noticeable drops. Additional statistics on how many filtered answers remain in the LLM context for each value of $\gamma$ at $N=128$ on MATH-500 are reported in Appendix~\ref{app:gamma-stats}.

\paragraph{Choice of Reward Models. }
We also adapt an outcome reward model (ORM), \textbf{AceMath–7B–RM}, to assess the robustness of reward signals \citep{acemath2024}. Compared with the PRM, this provides a weaker reward signal since intermediate reasoning steps are not evaluated. Figure~\ref{fig:ablation_combined}(b) reports accuracies across generation budgets under this ORM configuration. We find that RF-SeqBoN still remains the best among the baselines, validating the effectiveness of the reward-filtering strategy under different reward models. However, unlike with PRM, our adapted ORM exhibits a minor performance drop when the budget is large. We suspect this phenomenon may be attributed to \textit{reward over-optimization} \citep{gao2023scaling, frick2024evaluate, huang2025best}, and how to incorporate the reward-filtering idea to mitigate reward over-optimization remains an interesting future direction.

\begin{figure*}[t!]
    \centering

    \begin{minipage}{0.32\linewidth}
        \centering
        \includegraphics[width=\linewidth]{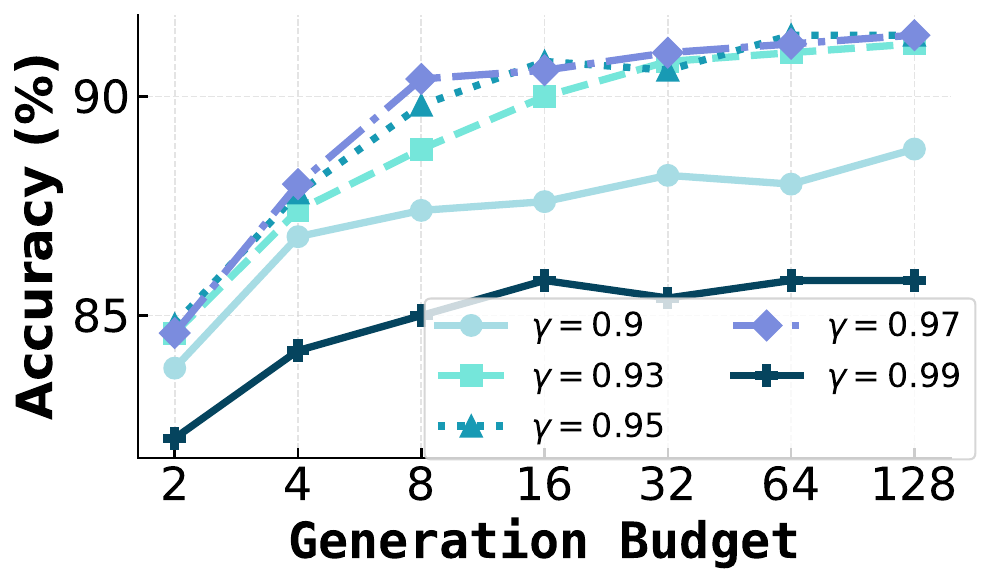}
        \caption*{(a) Ablation on $ \gamma $.}
        \label{fig:ablation_gamma}
    \end{minipage}
    \hfill
    \begin{minipage}{0.32\linewidth}
        \centering
        \includegraphics[width=\linewidth]{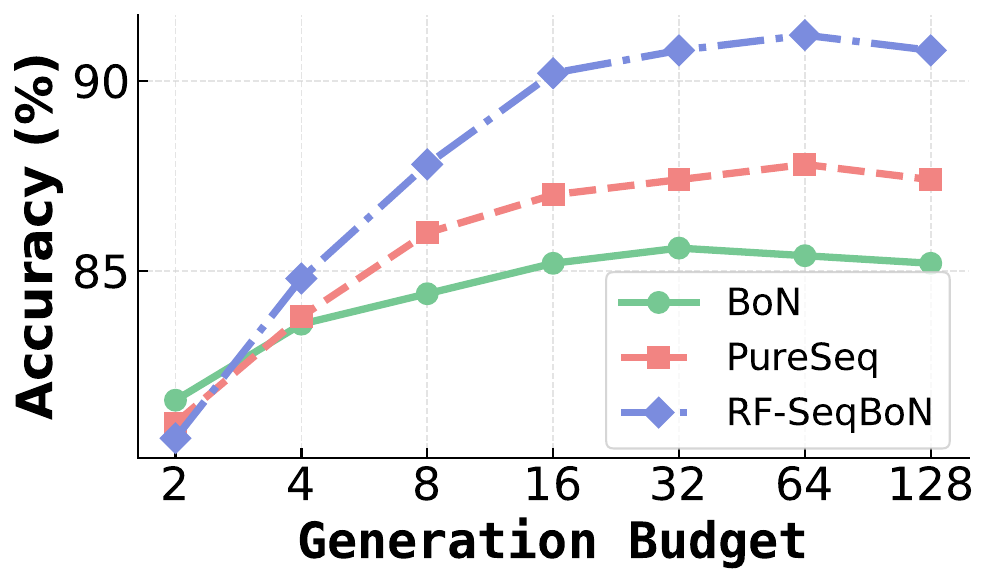}
        \caption*{(b) Ablation with ORM.}
        \label{fig:ablation_orm}
    \end{minipage}
    \hfill
    \begin{minipage}{0.32\linewidth}
        \centering
        \includegraphics[width=\linewidth]{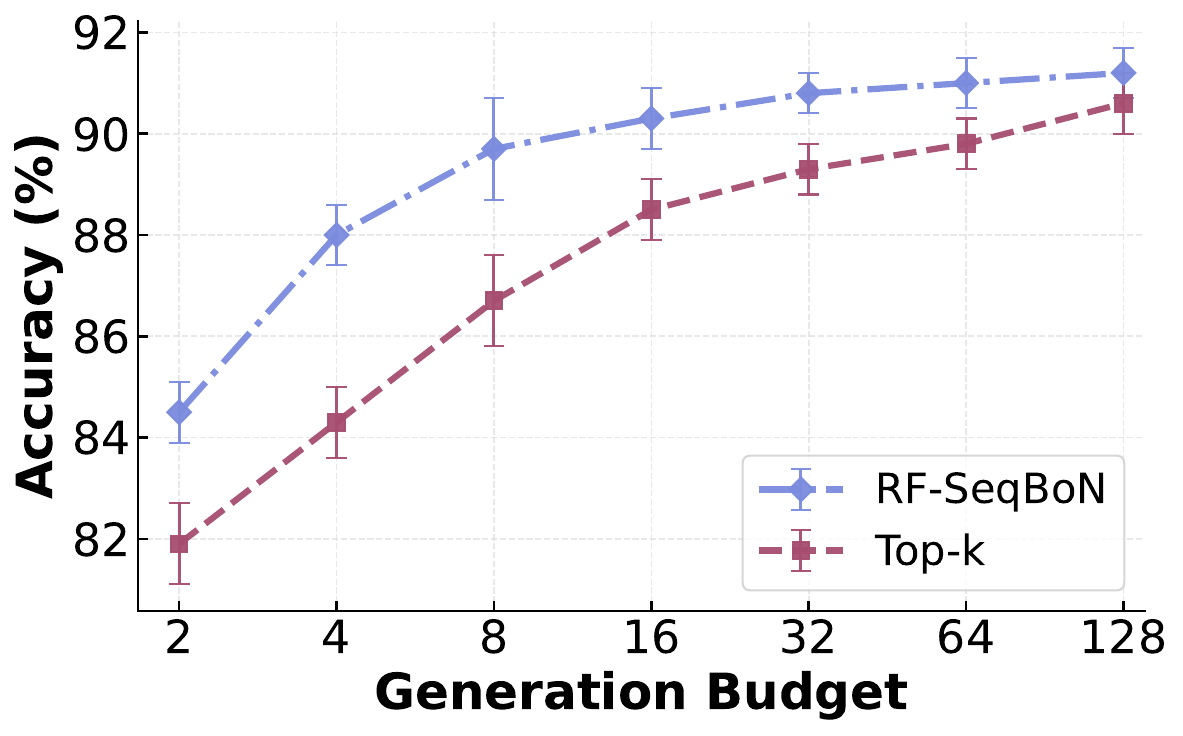}
        \caption*{(c) Ablation with Top-k.}
        \label{fig:ablation_topk}
    \end{minipage}

    \caption{Ablation study results on MATH500 using Qwen3-4B-Instruct as the foundation model.}
    \label{fig:ablation_combined}
\end{figure*}

\paragraph{Comparison Against Top-k Algorithm. } 
In addition to the BoN and PureSeq baselines, we also conducted ablation study to compare RF-SeqBoN against a \textbf{Top-k} refinement strategy, which does not use an explicit reward threshold. Both methods can be viewed as instantiations of the SeqBoN meta-algorithm (Algorithm~\ref{alg:1}): at iteration $i$, the algorithm samples a new answer from $\pi_\mathrm{LLM}\left( \cdot | h_i \right)$  and then updates the history based on the observed reward. For Top-k, we maintain all previously generated answers and their reward scores for a given question and construct $h_i$ by appending the top-$k$ responses ranked by reward, subject to the same sliding-window truncation used for PureSeq and RF-SeqBoN. To ensure a fair comparison, we set $k=3$, matching the \texttt{history\_budget=3} used for all sequential TTC algorithms, so that the two methods differ only in whether history is defined by a relative ranking (Top-k) or by crossing a fixed reward threshold $\gamma$ (RF-SeqBoN). The experiment is conducted on MATH500 dataset with Qwen3–4B–Instruct backbone model and the  Llama3.1–8B–PRM–Deepseek–Data reward model, using the same configuration as above. 

Figure~\ref{fig:ablation_combined}(c) reports accuracy as a function of generation budget $N$ for RF-SeqBoN and Top-k. Across all budgets, RF-SeqBoN consistently outperforms Top-k, with the margin being most pronounced in the low- and mid-budget regimes (small $N$) and gradually shrinking as $N$ increases. This pattern aligns with the intuition that, when the generation budget is limited, Top-k is prone to retaining several “locally good” but in fact incorrect trajectories: on difficult problems where genuinely high-reward responses are rare, the top-$k$ pool may consist entirely of low-quality yet relatively better generations. Once such spurious but high-ranked solutions are repeatedly recycled into the context, the method effectively degenerates toward PureSeq, reinforcing an erroneous reasoning trace instead of exploring alternative solution paths. By contrast, RF-SeqBoN discards \emph{all} responses whose reward falls below $\gamma$, preventing these degenerate feedback loops and ensuring that only truly high-quality trajectories are used as in-context exemplars. This selective reuse appears particularly beneficial when $N$ is small, precisely the setting where every history slot must be used most judiciously.

\section{Conclusion}
\label{sec:conclusion}
We presented a principled framework for TTC under a mixture-of-policies model and derived fundamental limits on the budget required for near-optimal performance. Our analysis shows that parallel TTC strategies such as BoN are inherently suboptimal. To address this, we proposed \emph{Reward-Filtered Sequential Best-of-$n$ (RF-SeqBoN)}, which selectively incorporates high-reward generations and provably refines the distribution toward the optimal policy. Experiments across diverse benchmarks confirm that RF-SeqBoN offers both stronger theoretical guarantees and substantial empirical gains in budget efficiency.

\appendix 



\section{Proof in Section \ref{sec:general}}
In this section, we provide proofs in Section \ref{sec:general}.

\subsection{Proof of Proposition \ref{prop:bon}}
\begin{proof}[Proof of Proposition \ref{prop:bon}]
    Using Lemma F.1 in \citet{huang2025best}, for any $M > 0$, the regret of vanilla BoN can be upper bounded by
    \begin{align*}
        \text{Regret}(\text{BoN}, x,\pi^*) \le \cE_M(\pi^*, \pi_{\llm}) + \exp \Big(-\frac{n}{M}\cdot\big(1-\cE_M(\pi^*,\pi_{\llm})\big)\Big).
    \end{align*}
    When taking $M = M_{\llm}^{x,\epsilon}$, we have $\cE_M(\pi^*,\pi_{\llm}) \le \epsilon$. Therefore, taking $n = M_{\llm}^{x,\epsilon} \log (1/\epsilon)$ will lead to $\text{Regret}(\text{BoN}, x,\pi^*) \le 2\epsilon$. On the other hand, for any parallel TTC algorithm $A$, consider the following problem instance, which is the same as that in the proof of Theorem \ref{thm:lower}. Let the reward be defined as: 
    \begin{align}
r(a): = \left\{ \begin{aligned}
    1,\ & \pi^\star(a)\geq \pi_A(a),\notag \\
    -1,\  &  \pi^\star(a)< \pi_A(a).\notag \\
\end{aligned}\right.
\end{align}
Therefore, the regret of algorithm $A$ can be represented as the TV distance between $\pi^*$
 and $\pi_A$, i.e.,
\begin{align}
\text{Regret}(A; x, \pi^\star) &= \EE_{a \sim \pi^\star(\cdot)} r(a) - \EE_{a \sim \pi_A(\cdot )} r(a)\notag \\
& = \sum_{a\in \cA} r(a)(\pi^\star(a) - \pi_A(a))\notag \\
& = \sum_{a\in \cA} |\pi^\star(a) - \pi_A(a)|.\label{eq:aa01}
\end{align}
Using Theorem D.6 in \citet{huang2025best}, when the number of samples $n \le M_{\llm}^{x,\epsilon}$, the TV distance between $\pi^*$ and $\pi_A$ can be lower bounded by $\epsilon$.
\begin{align}
\label{eq:aa02}
    \text{TV}\big(\pi^*(\cdot|x), \pi(\cdot|x)\big) \ge \epsilon.
\end{align}
Combining \eqref{eq:aa01} and \eqref{eq:aa02}, we have
\begin{align*}
    \text{Regret}(A; x, \pi^\star) \ge \epsilon.
\end{align*}
This completes the proof of Proposition \ref{prop:bon}.
\end{proof}

\subsection{Proof of Theorem \ref{thm:lower}}
We omit $x$ in the following proof. Denote $\tau$ such that $\tau$ is the index that takes the maximum action with highest probability, i.e., 
\begin{align}
    \tau^\star(x) \in \arg\max_{\tau\in \dist_{\refe}} \log \pi_{\refe}^\tau(a^\star(x) \mid x).\label{eq:aaa}
    \end{align}
Then fix any $M$, denote 
\begin{align}
\cA_M:=\{a: \pi^\star(a)\geq M\cdot \pi^\tau(a)\}. 
\end{align}

\begin{proof}[Proof of Theorem \ref{thm:lower}]
For the first result, we set the reward $r$ as 
\begin{align}
r(a): = \left\{ \begin{aligned}
    1,\ & \pi^\star(a)\geq \pi_A(a),\notag \\
    -1,\  &  \pi^\star(a)< \pi_A(a),\notag \\
\end{aligned}\right.
\end{align}
Then we have 
\begin{align}
\text{Regret}(A; x, \pi^\star) &= \EE_{a \sim \pi^\star(\cdot)} r(a) - \EE_{a \sim \pi_A(\cdot )} r(a)\notag \\
& = \sum_{a\in \cA} r(a)(\pi^\star(a) - \pi_A(a))\notag \\
& = \sum_{a\in \cA} |\pi^\star(a) - \pi_A(a)|.\notag
\end{align}
In this example, the regret of algorithm $A$ with respect to $\pi^\star$ is equivalent to the total variation distance between $\pi^\star$ and $\pi_A$.
Next, if $n < M^{x, \epsilon}_{\tau^\star(x)}$, we have 
\begin{align}
\sum_{a\in \cA} |\pi^\star(a) - \pi_A(a)| &\geq \sum_{a \in \cA_M}| \pi^\star(a) - \pi_A(a)|\notag \\
& \geq \sum_{a \in \cA_M}\pi^\star(a) - \EE_{a_i \sim p_i}\PP(a \in \{a_i\})\notag\\
& \geq \sum_{a \in \cA_M}\pi^\star(a) - n\cdot \pi_\refe^\tau(a \mid x)\notag\\
& = \mathcal{E}_{n}(\pi^\star, \pi_\refe^\tau(\cdot \mid x)),
\end{align}
where the second inequality we use the fact that $\pi_A$ can only output $a$ when $a \in \{a_i\}$, $p_i(\cdot) = \pi_{\llm}(\cdot|h_i)$.
The third one holds due to the fact that for any $a \in \cA_M\subseteq \text{supp}(\pi^\star)$, we have $\pi^\tau_{\text{ref}}(a)$ achieves the maximum probability over $\tau \in \dist_{\refe}$, i.e., $\pi^\tau_{\text{ref}}(a) \ge p_i(a)$. Therefore, recall that $n < M^{x, \epsilon}_{\tau^\star(x)}$, and $M^{x, \epsilon}_{\tau^\star(x)}$ is the smallest $M$ such that $\mathcal{E}_{M}(\pi^\star, \pi_\refe^\tau(\cdot \mid x)) \le \epsilon$. We have $\text{Regret}(A; x, \pi^\star) \geq \epsilon$. 

For the second result, we show that $M^{x, \epsilon}_{\tau^\star(x)} \leq M^{x, \epsilon}_\llm$. By definition of the $\cE_M$-divergence (Definition~\ref{def:cEM}), we have 
\begin{align}
\cE_{M^{x, \epsilon}_\llm}(\pi^\star(\cdot \mid x), \pi_\refe^{\tau^\star(x)}(\cdot \mid x))&: = \sum_{a \in \cA} \max \{0, \pi^\star(\cdot \mid x) - M^{x, \epsilon}_\llm\pi_\refe^{\tau^\star(x)}(\cdot \mid x)\}\notag \\
& \leq \sum_{a \in \cA} \max \{0, \pi^\star(\cdot \mid x) - M^{x, \epsilon}_\llm \pi_\llm(\cdot \mid x)\}\notag \\
& \leq \epsilon,\notag
\end{align}
where the first inequality holds due to \eqref{eq:aaa}. The second inequality holds due to the definition of $M^{x, \epsilon}_\llm$.
Using again $M^{x, \epsilon}_{\tau^\star(x)}$ is the smallest $M$ such that $\mathcal{E}_{M}(\pi^\star(\cdot \mid x), \pi_\refe^\tau(\cdot \mid x)) \le \epsilon$, we draw the final conclusion. 
\end{proof}

\section{Proof in Section \ref{sec:newalg}}

We first prove Theorem \ref{thm:general_bon}. In order to prove that, we first propose a generalized version of rejection

\subsection{Proof of Theorem \ref{thm:general_bon}}
\label{sec:pf_bon}

\begin{algorithm}[t!]
\caption{Adaptive Rejection Sampling ($\texttt{RejectionSampling}_{n}((w_i,p_i,M_i)_{i=1}^n, x)$)}\label{alg:rej}
\begin{algorithmic}[1]
\REQUIRE Prompt $x$, sampling budgets $n$, adaptive sampling policies $(p_i)_{i=1}^n$, importance weights $(w_i)_{i=1}^n$, truncation level $(M_i)_{i=1}^n$.
\STATE Sample $(a_1, \dots, a_{n}, a_{n+1})$ satisfying $a_i \sim p_i$, where $p_i(\cdot)$ can depend on $a_1,\dots, a_{i-1}$. 
\FOR{$i = 1\dots n$}
    \STATE Sample a Bernoulli random variable $\xi_i$ such that 
    $\mathbb{P}(\xi_i = 1 \mid a_i) = q_i, q_i:=\min\left\{ \frac{w_i(a_i)}{M_i}, 1 \right\}$, where $w_i, M_i$ can depend on $a_1,\dots, a_{i-1}$. 
    \IF{$\xi_i = 1$}
        \RETURN $a_i$
    \ENDIF
\ENDFOR
\RETURN  $ a_{n+1}$
\end{algorithmic}
\end{algorithm}

First, we consider the generalized version of rejection sampling in Algorithm \ref{alg:rej} as an auxiliary policy. Specifically, we set the importance weights $w_i(a) := \pi^{\star}(a)/p_i(a)$. The result is formulated in the following theorem.
\begin{theorem}\label{thm:general_rej}
For any comparator policy $\pi^{\star}$, let $\pi_{R}$ denote the distribution induced by the adaptive rejection sampling algorithm $\texttt{RejectionSampling}_{n}((w_i,p_i,M_i)_{i=1}^n, x)$ (Algorithm \ref{alg:rej}) with importance weights $w_i(a) := \pi^{\star}(a)/p_i(a)$. Then the total variation distance between $\pi^{\star}$ and $\pi_R$ can be upper bounded by:
\[
\mathrm{TV}(\pi^\star,\pi_R) \;\le\; \sum_{i=1}^{n} \alpha_i \mathbb{E}_{a_1,\dots,a_n}[\mathcal{E}_i] \;+\; 2\alpha_{n+1},
\]
where $\cE_i := \cE_{M_i}(\pi^{\star}, p_i)$, and $(\alpha_i)_{i=1}^{n+1}$ is a sequence of positive weights satisfying $\sum_{i=1}^{n+1} \alpha_i = 1$. 
Furthermore, we have
\begin{align*}    \alpha_{n+1} \le \EE_{a_1,\dots,a_N} \bigg[\exp\bigg(-\sum_{i=1}^n \frac{1}{M_i}\bigg)\bigg].
\end{align*}
If we set 
$M_i = M^\epsilon_{\pi^\star(\cdot, x), \pi_\llm(\cdot \mid h_i)}$, we have $\cE_i \le \epsilon$ due to Definition \ref{def:cEM}. Then the $TV$ distance can be bounded by 
\begin{align}
\mathrm{TV}(\pi^\star,\pi_R) \leq \epsilon + \EE_{a_1,\dots,a_n} \bigg[\exp\bigg(-\sum_{i=1}^n \frac{1}{M^\epsilon_{\pi^\star(\cdot \mid x), \pi_\llm(\cdot \mid h_i)}}\bigg)\bigg].
\end{align}
\end{theorem}

\begin{proof}
    For simplicity, we fix $x$ and ignore all the $x$ dependence. We denote $p_i(\cdot): = \pi_\llm(\cdot \mid h_i)$. 
Following the proof of Lemma D.4 in \citet{huang2025best}, we consider the truncated pseudo-distribution
\[
\pi_i(a) := \min\{\pi^\star(a),\, M_i p_i(a)\},
\]
and $A_i := \sum_a \pi_i(a) \le 1$ is the total mass of $\pi_i$. Moreover, we have
\begin{align*}
    A_i &= \sum_a \min\{\pi^\star(a),\, M_i p_i(a)\}\\
    & = \sum_{a \notin \cA_i} \pi^{\star}(a) + \sum_{a\in \cA_i} M_i p_i(a)\\
    & = 1 - \sum_{a \in \cA_i} \Big[\pi^{\star}(a) - M_i p_i(a)\Big].
\end{align*}
where $\cA_{i} = \{a \in \cA: \pi^*(a) > M_i p_i(a)\}$. Recalling the definition of $\cE_{M_i}(\pi^{\star},p_i)$ in Definition \ref{def:cEM}, we have

\[\quad 
\mathcal{E}_i := \cE_{M_i}(\pi^{\star}, p_i) = 1 - A_i.\]

Following the proof of Lemma D.4 in \citet{huang2025best}, we have
\begin{align}
\notag
\PP_{a' \sim p_i}(a' = a \mid \xi_i = 1) = \frac{\pi_i(a)}{A_i}. 
\end{align}

From now on, we begin to prove the original statement. 
For $a_i \sim p_i$, let 
$q_i := \min\Bigl\{ \tfrac{\pi^\star(a_i)}{M_i p_i(a_i)},\,1\Bigr\}$ be the probability of acceptance at step $i$.
We define
\[
\alpha_i := \mathbb{E}_{a_1,\dots,a_i}\big[(1-q_1)\cdots(1-q_{i-1})q_i\big], 
\quad 
\alpha_{n+1} := \mathbb{E}_{a_1,\dots,a_n}\big[(1-q_1)\cdots(1-q_n)\big],
\]
so that $\sum_{i=1}^{n+1}\alpha_i = 1$. 
Let $i^\star$ denote the index at which Algorithm~\ref{alg:rej} stops, and let $\hat a$ denote the output.  

Hence,
\begin{align}
\notag
\pi_R(a) &= \PP(\hat a = a)\\ \label{eq:qw0001}
&= \underbrace{\sum_{i=1}^N \PP(\hat a = a \mid i^\star = i)\PP(i^\star = i)}_{T_1}
   + \underbrace{\PP(\hat a = a \mid i^\star = n+1)\PP(i^\star = n+1)}_{T_2}. 
\end{align}

For the first term $T_1$, we compute
\begin{align}
T_1 
&= \sum_{i=1}^n \EE_{a_1,\dots,a_{i-1}}
   \PP_{a' \sim p_i}(a'=a \mid \xi_i=1)\,
   \EE_{a_1,\dots,a_n}\PP(i^\star=i \mid a_1,\dots,a_n) \notag \\\notag
&= \sum_{i=1}^n \EE_{a_1,\dots,a_{i-1}} \frac{\pi_i(a)}{A_i}
   \cdot \underbrace{\EE_{a_1,\dots,a_i}(1-q_1)\cdots(1-q_{i-1})q_i}_{\alpha_i} \\\notag
&= \sum_{i=1}^n \alpha_i \EE_{a_1,\dots,a_n}\frac{\pi_i(a)}{A_i} \\\label{eq:qw0002}
&= \EE_{a_1,\dots,a_n}\bigg[\sum_{i=1}^n \alpha_i \frac{\pi_i(a)}{A_i}\bigg].
\end{align}

For the second term $T_2$, we have
\begin{align}
\notag
T_2 
&\le \PP(i^\star = n+1) \\\notag
&= \EE_{a_1,\dots,a_n}\PP(i^\star=n+1 \mid a_1,\dots,a_n) \\\label{eq:qw0003}
&= \underbrace{\EE_{a_1,\dots,a_n}(1-q_1)\cdots(1-q_n)}_{\alpha_{n+1}}.
\end{align}
Substituting \eqref{eq:qw0002} and \eqref{eq:qw0003} into \eqref{eq:qw0001}, we have
\begin{align*}
    |\pi^{\star}(a)-\pi_R(a)| &\le |\pi^{\star}(a) - T_1| + T_2\\
    & \le \bigg|\sum_{i=1}^{n+1} \alpha_i \pi^{\star}(a) - \EE_{a_1,\dots,a_n}\bigg[\sum_{i=1}^n \alpha_i \frac{\pi_i(a)}{A_i}\bigg]\bigg| + \alpha_{n+1}\\
    &\le \bigg|\sum_{i=1}^{n} \alpha_i \pi^{\star}(y) - \EE_{a_1,\dots,a_n}\bigg[\sum_{i=1}^n \alpha_i \frac{\pi_i(a)}{A_i}\bigg]\bigg| + 2\alpha_{n+1},
\end{align*}
where the second inequality uses $\sum_{i=1}^{n+1}\alpha_i = 1$. 
Now consider the total variation distance:
\begin{align}
\text{TV}(\pi^\star,\pi_R) 
&:= \sum_a \big|\pi^\star(a) - \pi_R(a)\big| \notag \\\notag
&\le \sum_a \bigg|\sum_{i=1}^n \alpha_i \pi^\star(a) 
        - \EE_{a_1,\dots,a_n}\bigg[\sum_{i=1}^n \alpha_i \frac{\pi_i(y)}{A_i}\bigg]\bigg|
        + 2\alpha_{n+1} \\\notag
        & \le \sum_a \bigg|\EE_{a_1,\dots,a_n}\bigg[ \sum_{i=1}^n \alpha_i \Big(\pi^\star(a) 
        - \tfrac{\pi_i(a)}{A_i}\Big)\bigg]\bigg|
        + 2\alpha_{n+1}\\\label{eq:qw0004}
&\le \sum_{i=1}^n \alpha_i \EE_{a_1,\dots,a_n}\bigg[\sum_a\Big|\pi^\star(a) - \tfrac{\pi_i(a)}{A_i}\Big|\bigg] 
    + 2\alpha_{n+1},
\end{align}
where the last inequality holds due to the triangle inequality.
Finally, for each $i$, we bound
\begin{align}
\notag
\sum_a \Big|\pi^\star(a) - \tfrac{\pi_i(a)}{A_i}\Big|
&\le \sum_a \big|\pi^\star(a) - \pi_i(a)\big| 
   + \sum_a \Big|\pi_i(a) - \tfrac{\pi_i(a)}{A_i}\Big| \\\notag
&= \sum_{a\in \cA_i} (\pi^\star(a) - M_ip_i(a)) 
   + \sum_a \Big(\tfrac{\pi_i(a)}{A_i} - \pi_i(a)\Big) \\\notag
&= \mathcal{E}_i + (1-A_i)\\\label{eq:qw0005}
& = 2\cE_i,
\end{align}
where $\cA_i = \{a \in \cA: \pi^*(a) > M_i p_i(a)\}$. The first equation holds due to $\pi_i(a) := \min\{\pi^\star(a),\, M_i p_i(a)\}$ and $A_i \le 1$. The second equation holds due to $\sum_a \pi_i(a) = A_i$. The last equation holds due to $\cE_i = 1-A_i$.
Substituting \eqref{eq:qw0005} into \eqref{eq:qw0004} yields
\begin{align}
\text{TV}(\pi^\star,\pi_R) \le \sum_{i=1}^n \alpha_i \mathbb{E}_{a_1,\dots,a_n}[\mathcal{E}_i] + 2\alpha_{n+1}.
\end{align}
Moreover, we have
\begin{align}
\label{eq:qw0006}
   \alpha_{n+1} = \EE_{a_1,\dots,a_n}\bigg[\prod_{i=1}^n \bigg(1 - \frac{A_i}{M_i}\bigg)\bigg] =\EE_{a_1,\dots,a_n}\bigg[\prod_{i=1}^n \bigg(1 - \frac{1-\mathcal{E}_i}{M_i}\bigg)\bigg].
\end{align}
Finally, setting $M_i =M^\epsilon_{\pi^\star(\cdot \mid x), \pi_\llm(\cdot \mid h_i)}$ we have $\cE_i \leq \epsilon$, and 
\begin{align}
\notag
\alpha_{n+1} &= \EE_{a_1,\dots,a_n}\bigg[\prod_{i=1}^n \bigg(1 - \frac{1-\mathcal{E}_i}{M_i}\bigg)\bigg]\\\label{eq:qw0007} 
&\leq \EE_{a_1,\dots,a_n} \bigg[\exp\bigg(-\sum_{i=1}^n \frac{1}{M_i}\bigg)\bigg]\notag \\
& \leq \EE_{a_1,\dots,a_n} \bigg[\exp\bigg(- \frac{n^2}{\sum_{i=1}^n M^\epsilon_{\pi^\star(\cdot \mid x), \pi_\llm(\cdot \mid h_i)}}\bigg)\bigg]
\end{align}
where we use $1-x \le \exp(-x), \forall x \in \RR$ and Definition \ref{def:cEM}. Combining \eqref{eq:qw0006} and \eqref{eq:qw0007}, we complete the proof of Theorem \ref{thm:general_rej}.
\end{proof}

Using this result, now we can prove Theorem \ref{thm:general_bon}.
\begin{proof}[Proof of Theorem \ref{thm:general_bon}]
let $\pi_R$ be the auxilliary distribution defined in Theorem~\ref{thm:general_rej}. To begin with, we have 
\begin{align}
&\EE_{a \sim \pi^\star}[r(a, x)] - \EE_{a \sim \pi_{\text{abon}}}[r(a, x)] \notag \\
&=\underbrace{\EE_{a \sim \pi^\star}[r(a, x)] - \EE_{a \sim \pi_R}[r(a, x)]}_{I_1} + \underbrace{\EE_{a \sim \pi_R}[r(a, x)] - \EE_{a \sim \pi_{\text{abon}}}[r(a, x)]}_{I_2}\notag .
\end{align}
For $I_1$, we have
\begin{align*}
    I_1 \le \mathrm{TV}(\pi^\star,\pi_R).
\end{align*}
For $I_2$, we have
\begin{align*}
    I_2 = \EE_{a_1,\dots, a_n}\big[\EE_{a \sim \pi_R \mid a_1,\dots, a_n} [r(a,x)] - \EE_{a \sim \pi_{\text{abon}} \mid a_1,\dots, a_n} [r(a,x)]\big].
\end{align*}
Note that for any fixed $a_1, \dots, a_n$, Algorithm~\ref{alg:1} will always return $a_i$ which achieves the maximum reward. Thus, $I_2$ is non-positive. Finally, we use Theorem \ref{thm:general_rej} to bound $\mathrm{TV}(\pi^\star,\pi_R)$ and thus we complete the proof of Theorem \ref{thm:general_bon}.  
\end{proof}

\subsection{Proof of Proposition \ref{prop:smooth}}\label{app:condition}

\begin{proposition}\label{prop:smooth}
    Suppose the following conditions hold for any $x$: there exists some embedding $\phi(a,x)$ such that
    \begin{itemize}[leftmargin = *]
        \item The reward function $r(\cdot, x)$ is smooth at $a^\star(x)$ w.r.t. $\phi(a,x)$.
        \item For any ${\tau\in \dist_{\refe}}$, the reference policy $\log\pi_{\refe}^\tau(\cdot \mid x)$ is $L$-smooth w.r.t. $\phi(a,x)$.
        \item There exists $\Delta(x) > 0$ such that $\log\pi_{\refe}^{\tau^\star(x)}(a^\star(x) \mid x) - \sup_{\tau \neq \tau^\star(x)} \log\pi_{\refe}^\tau(a^\star(x) \mid x) > \Delta(x).$
    \end{itemize}
    Then there exists a threshold $\gamma^\star(x) < 1$ such that Assumption~\ref{assumption:maxref} holds.
\end{proposition}

Proposition~\ref{prop:smooth} provides a sufficient condition for Assumption~\ref{assumption:maxref}. 
The smoothness of $r$ and $\pi_{\refe}$ ensures that in a neighborhood of $a^\star(x)$, both the reward value and the probability gap between $\tau^\star(x)$ and other reference policies vary continuously. 
Since the margin at $a^\star(x)$ is strictly positive ($>\Delta(x)$), this gap persists in a small neighborhood around $a^\star(x)$. 
Consequently, there exists a reward threshold $\gamma^\star(x) < 1$ such that any action with $r(a,x) \ge \gamma^\star(x)$ lies within this neighborhood and inherits the same probability margin. 
Intuitively, if the optimal action is well separated from all others in terms of generation probability, and both the reward and policy distributions change smoothly, this separation extends to all sufficiently high-reward actions.

\begin{proof}[Proof of Proposition \ref{prop:smooth}]
    Since $r$ is smooth on the point $a^\star(x)$, then there must exists $\gamma^*(x)$ such that
\begin{align}
r(a^\star(x), x) - r(a, x)<1-\gamma^*(x) \Rightarrow \|\phi(a,x) - \phi(a^\star(x), x)\| \leq \Delta/(4L).
\end{align}
Then, due to smoothness assumption on $\pi_{\refe}$, we have
\begin{align}
&\log\pi_{\refe}^{\tau^\star(x)}(a \mid x) - \sup_{\tau \neq \tau^\star(x)}\log\pi_{\refe}^\tau(a \mid x) \notag \\[4pt]
& \geq \log \pi_{\refe}^{\tau^\star(x)}(a^\star(x) \mid x) 
   - L\|\phi(a,x) - \phi(a^\star(x),x)\| \notag \\ 
&\quad - \sup_{\tau \neq \tau^\star(x)}\Big(\log \pi_{\refe}^\tau(a^\star(x) \mid x) \Big) 
   - L\|\phi(a,x) - \phi(a^\star(x),x)\| \notag \\[4pt]
&>\Delta - 2L\|\phi(a,x) - \phi(a^\star(x),x)\|\notag\\
&\geq \Delta/2.\notag
\end{align}
\end{proof}

\subsection{Proof of Theorem \ref{thm:reward_filter}}
We first need the following lemma modified from Lemma C.3 in \citet{hu2024unveiling}:
\begin{lemma}\label{lemma:error}
For any $i \in [n]$, let $k$ denote the length of $h_i$, i.e., the number of actions in $h_i$. Then we have 
\begin{align}
\frac{\pi_{\llm}(a\mid h_i)}{\pi_{\refe}^{\tau^\star(x)}(a \mid x)}  > \frac{1}{ 1+ e^{-k\Delta}/p_{\refe}(\tau^\star(x))}.\notag
\end{align}
\end{lemma}

\begin{proof}[Proof of Lemma \ref{lemma:error}]
    For simplicity, let $i_1 < \dots < i_k$ denote the indices of $a_i$ that have been put into $h$. 
By the product factorization and Assumption \ref{assumption:maxref},
\begin{align}
\label{eq:qw0008}
\log\frac{\PP(h_i\mid \tau)}{\PP(h_i\mid \tau^\star(x))}
= \sum_{j=1}^k \log\frac{\pi_{\refe}^\tau(a_{i_j} \mid x)}{\pi_{\refe}^{\tau^\star(x)}(a_{i_j} \mid x)}
\le \sum_{j=1}^k (-\Delta) = -k\Delta,
\end{align}
which yields the claim after exponentiation.

Next, by the assumption of the pretraining distribution $\PP_\cD$, the predictive policy of the LLM is the Bayes mixture
\begin{align}
\pi_{\llm}(\cdot\mid h_i)=\sum_{\tau} \pi_{\refe}^\tau(\cdot \mid x)\cdot w_i(\tau),
\qquad
w_i(\tau):=\frac{p_{\refe}(\tau)\PP(h_i\mid \tau)}{\sum_{\tau'} p_{\refe}(\tau')\PP(h_i\mid \tau')}.
\end{align}
Using the likelihood–ratio bound \eqref{eq:qw0008}, for any $\tau\neq \tau^\star(x)$,
\begin{align}
\frac{w_i(\tau)}{w_i(\tau^\star(x))}
=\frac{\PP(h_i\mid \tau)}{\PP(h_i\mid \tau^\star(x))}\cdot \frac{p_{\refe}(\tau)}{p_{\refe}(\tau^\star(x))}
\le e^{-k\Delta}\cdot \frac{p_{\refe}(\tau)}{p_{\refe}(\tau^\star(x))}.
\end{align}
Summing over $\{\tau: \tau \neq \tau^\star(x)\}$, we have
\begin{align*}
    \frac{1-w_m\big(\tau^\star(x)\big)}{w_m\big(\tau^\star(x)\big)} \le e^{-m\Delta}\cdot \frac{1-p_{\refe}(\tau^\star(x))}{p_{\refe}(\tau)}.
\end{align*}
Then it is easy to obtain that 
\begin{align}
w_i(\tau^\star(x)) \geq 1-\frac{1 - p_{\refe}(\tau^\star(x))}{1+p_{\refe}(\tau^\star(x))(e^{k\Delta} - 1)} > \frac{1}{ 1+ e^{-k\Delta}/p_{\refe}(\tau^\star(x))}
\end{align}
Therefore, for any action $a$, we have 
\begin{align}
\frac{\pi_{\llm}(a\mid h_i)}{\pi_{\refe}^{\tau^\star(x)}(a \mid x)} \geq  w_i(\tau^\star(x)) \geq \frac{1}{ 1+ e^{-k\Delta}/p_{\refe}(\tau^\star(x))}. 
\end{align}
\end{proof}

Next we have our overall proof of Theorem \ref{thm:reward_filter}.
\begin{proof}[Proof of Theorem \ref{thm:reward_filter}]

Denote $i^\star$ be the first index $i$ satisfying $|h_i| = m$. Here we allow $i^\star > n$, in which case $|h_i| < m$ for any $1\le i\le n$. By the bound established in Lemma \ref{lemma:error} and the selection of $m = {\log [\epsilon^{-1}p_\refe^{-1}(\tau^\star(x))]}/{\Delta}$, for any $i > i^\star$, we have 
\begin{align}
\label{eq:qw0009}
\frac{\pi_{\llm}(a\mid h_{i})}{\pi_{\refe}^{\tau^\star(x)}(a \mid x)} \geq  \frac{1}{ 1+ e^{-m\Delta}/p_\refe(\tau^\star(x))}>1-\epsilon. 
\end{align}

Let $M_0 = M^{\epsilon}_{\pi^\star(\cdot \mid x), \pi_{\refe}^{\tau^\star(x)}(\cdot\mid x)}$. Using the definition of $\cE_M$-divergence, we have
\begin{align}
&\cE_{M_0}\big(\pi^*(\cdot\mid x), \pi_{\llm}(\cdot\mid h_{i})\big)=\sum_{a \in \cA} \max\{0, \pi^\star(a \mid x) - M^{\epsilon}_{\pi^\star(\cdot \mid x), \pi_{\refe}^{\tau^\star(x)}(\cdot \mid x)}\cdot \pi_{\llm}(a\mid h_{i})\}\notag \\
& \leq \sum_{a \in \cA} \max\Big\{0, \pi^\star(a \mid x) - M^{\epsilon}_{\pi^\star(\cdot \mid x), \pi_{\refe}^{\tau^\star(x)}(\cdot \mid x)}\cdot (1-\epsilon)\cdot \pi_{\refe}^{\tau^\star(x)}(a \mid x)\Big\}\notag \\
& \leq (1-\epsilon)\cdot \sum_{a \in \cA} \max\Big\{0, \pi^\star(a \mid x) - M^{\epsilon}_{\pi^\star(\cdot \mid x), \pi_{\refe}^{\tau^\star(x)}(\cdot \mid x)}\cdot \pi_{\refe}^{\tau^\star(x)}(a \mid x)\Big\} + \epsilon \sum_{a \in \cA} \pi^\star(a \mid x)\notag \\
&\le\cE_{M_0}\big(\pi^*(\cdot\mid x), \pi_{\refe}^{\tau^\star(x)}(\cdot \mid x)\big) + \epsilon\notag\\
& \leq 2\epsilon,\notag
\end{align}
where the first inequality holds due to \ref{eq:qw0009}. The second inequality holds due to $\max \{a,b+c\} \le \max\{a,b\} + c$ for $a,b \in \RR$ and $c >0$. The third inequality holds due to $\cE_{M_0}(\pi^*, \pi_{\refe}^{\tau^\star(x)}) >0$. The last inequality holds because $M_0 =M^{\epsilon}_{\pi^\star(\cdot \mid x),\pi_{\refe}^{\tau^\star(x)}(\cdot\mid x)}$ is the smallest $M$ such that $\cE_{M}(\pi^*, \pi_{\refe}^{\tau^\star(x)}) \le \epsilon$. Therefore, we have
\begin{align*}
    M^{2\epsilon}_{\pi^\star(\cdot \mid x), \pi_{\llm}(\cdot\mid h_{i})} \leq M^{\epsilon}_{\pi^\star(\cdot \mid x), \pi_{\refe}^{\tau^\star(x)}(\cdot \mid x)}.
\end{align*}
Combined with the result of Theorem \ref{thm:lower}, we have for any $i > i^*$
\begin{align}
M^{2\epsilon}_{\pi^\star(\cdot \mid x), \pi_{\llm}(\cdot\mid h_{i})} \leq M^{\epsilon}_{\pi^\star(\cdot \mid x), \pi_{\refe}^{\tau^\star(x)}(\cdot \mid x)}\leq M^\epsilon_{\pi^\star(\cdot \mid x), \pi_{\llm}(\cdot\mid x)}.\label{help:3}
\end{align}

Note that for each step $i \leq i^\star$, the sampled action $a_i$ is accepted into into $h_i$ with probability $p_{\llm}$, defined as: 
\begin{align}
p_{\llm}:=\PP_{a \sim \pi_{\llm}(\cdot\mid x)} \big(r(a,x)\geq \gamma^\star(x)\big). \notag
\end{align}
Therefore, $i^\star$ satisfies a negative binomial distribution $N(m, p_{\llm})$, and its tail function $\PP(i^\star >k)$ for any $k$ is equal to $\PP(j^\star < m)$, where $j^\star$ satisfies a Binomial distribution $B(k, p_{\llm})$, which is 
\begin{align}
\PP(i^\star >k) = \PP\big(B(k, p_{\llm}) < m\big)  \leq \exp\bigg(-\frac{(m - kp_{\llm})^2}{kp_{\llm}(1-p_{\llm})}\bigg),
\end{align}
where we use the fact that the Binomial distribution is sub-Gaussian. Therefore, by selecting $k =\bar n=p_{\llm}^{-1}\cdot (3m + \log(1/\epsilon))$, we have 
\begin{align}
\PP(i^\star > \bar n) < \epsilon.\label{help:4}
\end{align}
Next, we analyze the original statement. 

By Theorem \ref{thm:general_bon}, we have 
\begin{align}
&\EE_{a \sim \pi^\star}[r(a, x)] - \EE_{a \sim \pi_{\text{reward}}}[r(a, x)] \notag \\
&\leq 2\epsilon + \EE_{a_1,\dots,a_n} \bigg[\exp\bigg(- \frac{n^2}{\sum_{i=1}^n M^{2\epsilon}_{\pi^\star(\cdot \mid x), \pi_{\llm}(\cdot\mid h_i)}}\bigg)\bigg]\notag \\
& = 2\epsilon + \EE_{a_1,\dots,a_n} \bigg[\exp\bigg(- \frac{n^2}{\sum_{i=1}^{i^\star} M^{2\epsilon}_{\pi^\star(\cdot \mid x), \pi_{\llm}(\cdot\mid h_i)} + \sum_{i=i^\star + 1}^{n} M^{2\epsilon}_{\pi^\star(\cdot \mid x), \pi_{\llm}(\cdot\mid h_i)}}\bigg)\bigg]\label{help:334} \\
& \leq 2\epsilon +\EE_{a_1,\dots,a_n} \bigg[\exp\bigg(- \frac{n^2}{\min\{i^\star, n\} M^{2\epsilon}_{\pi^\star(\cdot \mid x), \pi_{\llm}(\cdot\mid x)} + (n - \min\{i^\star, n\}) M^{\epsilon}_{\pi^\star(\cdot \mid x), \pi_{\refe}^{\tau^\star(x)}(\cdot \mid x)}}  \bigg)\bigg]\notag
\end{align}
where the last line holds due to \eqref{help:3}.

Next, we further decompose it based on the value of $i^*$: 
\begin{align}
    &\EE_{a \sim \pi^\star}[r(a, x)] - \EE_{a \sim \pi_{\text{reward}}}[r(a, x)]\notag \\
    &\leq 2\epsilon + \sum_{i=1}^\infty \PP(i^* = i)\bigg[\exp\bigg(- \frac{n^2}{\min\{i, n\} M^{2\epsilon}_{\pi^\star(\cdot \mid x), \pi_{\llm}(\cdot\mid x)} + (n - \min\{i, n\})\cdot M^{\epsilon}_{\pi^\star(\cdot \mid x), \pi_{\refe}^{\tau^\star(x)}(\cdot \mid x)}}  \bigg)\bigg]\notag \\
    &=2\epsilon + \bigg(\sum_{i=1}^{\bar n} + \sum_{i=\bar n + 1}^{\infty}\bigg)\PP(i^\star = i)\notag \\
    &\quad \cdot \bigg[\exp\bigg(- \frac{n^2}{\min\{i, n\} M^{2\epsilon}_{\pi^\star(\cdot \mid x), \pi_{\llm}(\cdot\mid x)} + (n - \min\{i, n\})\cdot M^{\epsilon}_{\pi^\star(\cdot \mid x), \pi_{\refe}^{\tau^\star(x)}(\cdot \mid x)}}  \bigg)\bigg]\notag \\
    & \leq 2\epsilon \notag \\
    &\quad + \sum_{i=1}^{\bar n}\PP(i^\star = i)\bigg[\exp\bigg(- \frac{n^2}{\min\{i, n\} M^{2\epsilon}_{\pi^\star(\cdot \mid x), \pi_{\llm}(\cdot\mid x)} + (n - \min\{i, n\})\cdot M^{\epsilon}_{\pi^\star(\cdot \mid x), \pi_{\refe}^{\tau^\star(x)}(\cdot \mid x)}}  \bigg)\bigg] \notag \\
    &\quad + \PP(i^\star >\bar n)\notag \\
    &\leq 2\epsilon \notag \\
    &\quad + \sum_{i=1}^{\bar n}\PP(i^\star = i)\bigg[\exp\bigg(- \frac{n^2}{\min\{i, n\} M^{\epsilon}_{\pi^\star(\cdot \mid x), \pi_{\llm}(\cdot\mid x)} + (n - \min\{i, n\})\cdot M^{\epsilon}_{\pi^\star(\cdot \mid x), \pi_{\refe}^{\tau^\star(x)}(\cdot \mid x)}}  \bigg)\bigg] \notag \\
    &\quad + \PP(i^\star >\bar n)\notag \\
    & \leq 3\epsilon + \exp\bigg(- \frac{n^2}{\bar n M^\epsilon_{\pi^\star(\cdot \mid x), \pi_{\llm}(\cdot\mid x)} + (n - \bar n)\cdot M^{\epsilon}_{\pi^\star(\cdot \mid x), \pi_{\refe}^{\tau^\star(x)}(\cdot \mid x)} } \bigg)\label{help:cont}
\end{align}
where the second inequality holds because the exponential term is less than 1. The third inequality holds due to $M^{2\epsilon}_{\pi^\star(\cdot \mid x), \pi_{\llm}(\cdot\mid x)} \le M^{\epsilon}_{\pi^\star(\cdot \mid x), \pi_{\llm}(\cdot\mid x)}$. The fourth inequality holds due to \eqref{help:4} and $\min \{i,n\} \le \bar n$, $M^\epsilon_{\pi^\star(\cdot \mid x), \pi_{\llm}(\cdot\mid x)} \ge M^{\epsilon}_{\pi^\star(\cdot \mid x), \pi_{\refe}^{\tau^\star(x)}(\cdot \mid x)}$. 

When $M^\epsilon_{\pi^\star(\cdot \mid x), \pi_{\llm}(\cdot\mid x)} \le \bar n$, as $M^\epsilon_{\pi^\star(\cdot \mid x), \pi_{\llm}(\cdot\mid x)} \ge M^{\epsilon}_{\pi^\star(\cdot \mid x), \pi_{\refe}^{\tau^\star(x)}(\cdot \mid x)}$, we have
\begin{align*}
    &\exp\bigg(- \frac{n^2}{\bar n M^\epsilon_{\pi^\star(\cdot \mid x), \pi_{\llm}(\cdot\mid x)} + (n - \bar n)\cdot M^{\epsilon}_{\pi^\star(\cdot \mid x), \pi_{\refe}^{\tau^\star(x)}(\cdot \mid x)} } \bigg)\\
    & \qquad \le \exp\bigg(- \frac{n}{ M^\epsilon_{\pi^\star(\cdot \mid x), \pi_{\llm}(\cdot\mid x)} } \bigg).
\end{align*}
Thus, it suffices to select $n = \text{poly} \log(1/\epsilon)\cdot M^\epsilon_{\pi^\star(\cdot \mid x), \pi_{\llm}(\cdot\mid x)}$.

Otherwise, using basic inequalities, we have
\begin{align*}
    &\sqrt{\log^2(1/\epsilon)\Big[M^{\epsilon}_{\pi^\star(\cdot \mid x), \pi_{\refe}^{\tau^\star(x)}(\cdot \mid x)}\Big]^2 + 4\log(1/\epsilon)\bar n \Big[M^\epsilon_{\pi^\star(\cdot \mid x), \pi_{\llm}(\cdot\mid x)} - M^{\epsilon}_{\pi^\star(\cdot \mid x), \pi_{\refe}^{\tau^\star(x)}(\cdot \mid x)}\Big]} \\
    & \le \log(1/\epsilon)M^{\epsilon}_{\pi^\star(\cdot \mid x), \pi_{\refe}^{\tau^\star(x)}(\cdot \mid x)}+ 2\sqrt{\log(1/\epsilon)} \cdot \sqrt{\bar n \Big[M^\epsilon_{\pi^\star(\cdot \mid x), \pi_{\llm}(\cdot\mid x)} - M^{\epsilon}_{\pi^\star(\cdot \mid x), \pi_{\refe}^{\tau^\star(x)}(\cdot \mid x)}\Big]},
\end{align*}
where we use the inequality $\sqrt{a+b} \le \sqrt{a} + \sqrt{b}$ for $a,b >0$. 
Therefore, when 
\begin{align*}
    n = \log(1/\epsilon)M^{\epsilon}_{\pi^\star(\cdot \mid x), \pi_{\refe}^{\tau^\star(x)}(\cdot \mid x)} + 2\sqrt{\log(1/\epsilon)} \cdot \sqrt{\bar n \Big[M^\epsilon_{\pi^\star(\cdot \mid x), \pi_{\llm}(\cdot\mid x)} - M^{\epsilon}_{\pi^\star(\cdot \mid x), \pi_{\refe}^{\tau^\star(x)}(\cdot \mid x)}\Big]},
\end{align*} 
we have
\begin{align*}
    n^2 \ge \log(1/\epsilon) n + \log(1/\epsilon)\bar n \Big[M^\epsilon_{\pi^\star(\cdot \mid x), \pi_{\llm}(\cdot\mid x)} - M^{\epsilon}_{\pi^\star(\cdot \mid x), \pi_{\refe}^{\tau^\star(x)}(\cdot \mid x)}\Big].
\end{align*}
Thus, $\EE_{a \sim \pi^\star}[r(a, x)] - \EE_{a \sim \pi_{\text{reward}}}[r(a, x)] \le 4\epsilon$. It suffices to select 
\begin{align*}
    n = \text{poly} \log(1/\epsilon) \Big[M^{x, \epsilon}_{\tau^\star(x)} + \sqrt{\bar n \big(M^{x, \epsilon}_{\llm} - M^{x, \epsilon}_{\tau^\star(x)}\big)}\Big].
\end{align*}
\end{proof}

\subsection{Proof of Corollary \ref{coro:1}}
\begin{proof}[Proof of Corollary \ref{coro:1}]
We follow the proof of Theorem \ref{thm:reward_filter} from \eqref{help:cont}. By Definition \ref{def:cEM}, using the fact that 
\begin{align}
M^\epsilon_{\pi^\star(\cdot, x), \pi_{\llm}(\cdot\mid x)} \leq \frac{C(\pi^\star(\cdot, x), \pi_{\llm}(\cdot\mid x))}{\epsilon},\ M^\epsilon_{\pi^\star(\cdot, x), \pi_{\refe}(\cdot, x; \tau^\star(x))} \leq \frac{C(\pi^\star(\cdot, x), \pi_{\refe}(\cdot, x; \tau^\star(x)))}{\epsilon},\notag
\end{align}
Next, we bound $C(\pi^\star(\cdot, x), \pi_{\refe}(\cdot, x; \tau^\star(x)))$. Note that 
\begin{align}
&C(\pi^\star(\cdot, x), \pi_{\refe}(\cdot, x; \tau^\star(x))) \notag \\
&= \EE_{a \sim \pi^\star(\cdot, x)}\frac{\pi^\star(a, x)}{\pi_{\refe}(a, x; \tau^\star(x))}\notag \\
& = \EE_{a \sim \pi^\star(\cdot, x)}\frac{\pi^\star(a, x)}{\pi_{\llm}(a\mid x)} \cdot \frac{\pi_{\llm}(a\mid x)}{\pi_{\refe}(a, x; \tau^\star(x))}\notag \\
& = \EE_{a \sim \pi^\star(\cdot, x)}\frac{\pi^\star(a, x)}{\pi_{\llm}(a\mid x)} \cdot \frac{\sum_{\tau \in \dist_\refe} p_\refe(\tau) \pi_{\refe}(a, x; \tau)}{\pi_{\refe}(a, x; \tau^\star(x))}\notag \\
& \leq \EE_{a \sim \pi^\star(\cdot, x)}\frac{\pi^\star(a, x)}{\pi_{\llm}(a\mid x)} \cdot \frac{\sum_{\tau \neq \tau^\star(x)} p_\refe(\tau) e^{-\Delta(x)} \pi_{\refe}(a, x; \tau^\star(x)) + p_\refe(\tau^\star(x)) \pi_{\refe}(a, x; \tau^\star(x))}{\pi_{\refe}(a, x; \tau^\star(x))}\notag \\
& \leq \min \{e^{-\Delta(x)}, p_\refe(\tau^\star(x)) \} \cdot C(\pi^\star(\cdot, x), \pi_{\llm}(\cdot\mid x)),\label{help:002} 
\end{align}
where the first inequality holds due to Assumption \ref{assumption:maxref} and $\text{supp}(\pi^\star(\cdot, x)) \subseteq \{a: \log \pi_{\refe}(a, x; \tau^\star(x)) - \sup_{\tau \neq \tau^\star(x)} \log \pi_{\refe}(a, x; \tau) \geq \Delta(x)\}$, the last one holds by calculation. Therefore, by selecting $\epsilon$ to be small enough, we have our statement by  \eqref{help:cont}.

\end{proof}

\subsection{Proof of Corollary \ref{coro:adaptive}}
\begin{proof}[Proof of Corollary \ref{coro:adaptive}]
Using Lemma \ref{lemma:error}, we have 
\begin{align}
\frac{\pi_{\llm}(a\mid h_i)}{\pi_{\refe}(a,x;\tau^\star(x))}  > \frac{1}{ 1+ e^{-k\Delta(x)}/p_{\refe}(\tau^\star(x))},\label{help:001}
\end{align}
where $k$ is the length of $h_i$. Following the proof of Corollary \ref{coro:1}, we have 
\begin{align}
&C(\pi^\star(\cdot, x), \pi_{\llm}(a\mid h_i)) \notag \\
&\leq (1+ e^{-k\Delta(x)}/p_{\refe}(\tau^\star(x)))C(\pi^\star(\cdot, x), \pi_{\refe}(\cdot, x; \tau^\star(x)))\notag \\
& \leq \underbrace{(1+ e^{-\Delta(x)}/p_{\refe}(\tau^\star(x)))}_{d(x)}C(\pi^\star(\cdot, x), \pi_{\refe}(\cdot, x; \tau^\star(x)))\notag\\
& \leq d(x)\cdot \kappa(x) \cdot C(\pi^\star(\cdot, x), \pi_{\llm}(\cdot\mid x)),\notag
\end{align}
where the first inequality holds due to \eqref{help:001}, the last one holds due to \eqref{help:002}.

We still denote $i^\star$ be the first index satisfying $|h_i| = m$. Then $i^\star$ is a random variable as follows: it is a summation of a series of Geometric distribution. Then we have 
\begin{align}
i^\star = G(p_1) + \dots + G(p_m), p_i = \PP_{a \sim \pi_\llm(\cdot \mid h_i)}(r(a,x) \geq \gamma^\star(x)). 
\end{align}
We know that each $G(p_i)$ is an $(1/p_i^2, 1/p_i)$-sub-exponential distribution. Using Bernstein inequality, we know that with probability at least $1-\epsilon$, we have 
\begin{align}
i^\star &\leq 2\bigg(\sum_{i=1}^m \frac{1}{p_i} + \sqrt{\sum_{i=1}^m \frac{1}{p_i^2}\log \epsilon^{-1}} + \frac{\log \epsilon^{-1}}{\min\{p_i\}}\bigg)\notag \\
& \leq \underbrace{\frac{4 m d(x) \log \epsilon^{-1}}{\PP_{a \sim \pi_{\refe}(\cdot ,x;\tau^\star(x))}(r(a,x) \geq \gamma^\star(x))}}_{\bar n}, \label{help:222} 
\end{align}
where we use the fact that 
\begin{align}
p_i &= \PP_{a \sim \pi_\llm(\cdot \mid h_i)}(r(a,x) \geq \gamma^\star(x))\notag \\
& = \EE_{a \sim \pi_{\refe}(\cdot ,x;\tau^\star(x))} \frac{\pi_{\llm}(a\mid h_i)}{\pi_{\refe}(a,x;\tau^\star(x))} \ind(r(a,x) \geq \gamma^\star(x))\notag \\
& \geq \frac{1}{ d(x)} \PP_{a \sim \pi_{\refe}(\cdot ,x;\tau^\star(x))}(r(a,x) \geq \gamma^\star(x)).\notag
\end{align}
Hence, selecting $\bar n$ as \eqref{help:222} suggests, following the same step from \eqref{help:334}, we have 
\begin{align}
&\EE_{a \sim \pi^\star}[r(a, x)] - \EE_{a \sim \pi_{\text{reward}}}[r(a, x)]\notag \\
& \leq \epsilon + \EE_{a_1,\dots,a_n} \bigg[\exp\bigg(- \frac{n^2}{\sum_{i=1}^{i^\star} M^\epsilon_{\pi^\star(\cdot, x), \pi_{\llm}(\cdot\mid h_i)} + \sum_{i=i^\star + 1}^{n} M^\epsilon_{\pi^\star(\cdot, x), \pi_{\llm}(\cdot\mid h_i)}}\bigg)\bigg]\notag \\
& \leq \epsilon + \EE_{a_1,\dots,a_n} \bigg[\exp\bigg(- \frac{n^2\epsilon}{\sum_{i=1}^{i^\star} C(\pi^\star(\cdot, x), \pi_{\llm}(\cdot\mid h_i)) + \sum_{i=i^\star + 1}^{n} C(\pi^\star(\cdot, x), \pi_{\llm}(\cdot\mid h_i))}\bigg)\bigg]\notag \\
& \leq \epsilon + \EE_{a_1,\dots,a_n} \bigg[\exp\bigg(- \frac{\frac{n^2\epsilon}{\kappa(x)\cdot C(\pi^\star(\cdot, x), \pi_{\llm}(\cdot\mid x))}}{\sum_{i=1}^{i^\star} d(x) + \sum_{i=i^\star + 1}^{n} 1}\bigg)\bigg]\notag \\
& \leq \epsilon + \EE_{a_1,\dots,a_n} \bigg[\exp\bigg(- \frac{\frac{n^2\epsilon}{\kappa(x)\cdot C(\pi^\star(\cdot, x), \pi_{\llm}(\cdot\mid x))}}{\min\{\bar n, n\} d(x) + (n - \min\{\bar n, n\})}\bigg)\bigg]\label{help:889}
\end{align}
Finally, selecting 
\begin{align}
\epsilon \leq \min\Bigg\{1-\gamma^\star(x), 
\frac{C^\star_{\llm}(x)\PP_{a \sim \pi_{\refe}(\cdot ,x;\tau^\star(x))}(r(a,x) \geq \gamma^\star(x))\Delta(x)}{\log p_\refe^{-1}(\tau^\star(x))}\cdot \frac{\kappa(x)}{d^2(x)} \Bigg\}, \notag
\end{align}
satisfies \eqref{help:889} to be $2\epsilon$-optimal. 
\end{proof}

\section{Experiments}
\label{app:experiments}
\subsection{Hyperparameters and Dataset- Specific Settings}
\label{app:hyperparameters}
Unless otherwise specified, decoding settings are identical across methods and datasets. Table~\ref{tab:hyperparameters} lists the global settings held fixed; Table~\ref{tab:dataset_specific_parameters} provides per-dataset limits (maximum output tokens and maximum context length) and the reward-filter threshold \(\gamma\) used with the PRM Llama3.1-8B-PRM-Deepseek-Data in the main results (Figure~\ref{fig:main}). Token budgets are adjusted by dataset to accommodate problem difficulty and solution length.

\begin{table}[h]
\centering
\caption{Key hyperparameters of experiments across all configurations.}
\label{tab:hyperparameters}
\footnotesize
\setlength{\tabcolsep}{6pt}
\begin{tabular}{lll}
\toprule
\textbf{Flag / Name} & \textbf{Value} & \textbf{Description} \\
\midrule
\texttt{--history\_budget} & 3 & Max number of recent solutions retained in history. \\
\texttt{--temperature}          & 0.8  & Sampling temperature of the foundation model. \\
\texttt{--top\_p}               & 0.95 & Nucleus sampling parameter. \\
\texttt{--prm\_agg}             & last-step & PRM aggregation strategy over reasoning steps. \\
\bottomrule
\end{tabular}
\end{table}

\begin{table}[h]
\centering
\caption{Dataset specific parameters used in main results (Figure~\ref{fig:main}).}
\label{tab:dataset_specific_parameters}
\footnotesize
\begin{tabular}{l|ccc}
\toprule
\textbf{Dataset} & \textbf{Max output tokens} & \textbf{Max context length} & \textbf{Reward- filter threshold} \(\gamma\) \\
\midrule
MATH500 & 2048 & 8192 & 0.97 \\
GPQA-Diamond & 4096 & 16384 & 0.92 \\
AMC\textquoteright 23 & 8192 & 32768 & 0.92 \\
AIME\textquoteright 24 & 8192 & 32768 & 0.95 \\
AIME\textquoteright 25 & 8192 & 32768 & 0.90 \\
\bottomrule
\end{tabular}
\end{table}

\subsection{Prompt Templates}
\label{sec:prompt-template}

We use a unified chat-style prompt across all methods and datasets. Each query is formatted as a sequence of \texttt{(system, user, assistant)} messages: a dataset-specific \emph{system} instruction, the \emph{user} problem text, and (optionally) a short \emph{assistant} history of previously accepted solutions.

\vspace{4pt}
\paragraph{Message construction (used by PureSeq and RF-SeqBoN).  }
When history is present, we add a brief self-critique instruction before requesting a fresh, complete solution and a single final-answer line. The high-level template is:

\begin{lstlisting}[language={},basicstyle=\ttfamily\small,frame=single]
messages = [
  {"role": "system", "content": SYSTEM_PROMPT(dataset)}, 
  {"role": "user",   "content": PROBLEM_TEXT},

  # If history exists, include up to # history_budget most recent
  # accepted solutions:
  {"role": "assistant", "content": ACCEPTED_{-k}}, {"role": "user",
  "content": "Try again with a different approach:"},
  ... (repeat for up to # history_budget earlier accepted solutions) ...
  {"role": "assistant", "content": ACCEPTED_{-1}},
  {"role": "user", "content":
     "The previous solution(s) may contain errors. 
      Before solving, briefly critique the previous attempt(s) 
      in 2 to 3 bullet points. 
      Then provide a COMPLETE and CONCISE corrected solution 
      from scratch that addresses those issues. 
      End with exactly one line containing the final answer:"}
]
\end{lstlisting}

\paragraph{Dataset-specific \texttt{system\_prompt}.  }
\begin{itemize}[leftmargin=1.2em]
\item \textbf{MATH500.} Concise for simple items; step-by-step sections for harder ones; end with “\(\boxed{\text{answer}}\)”.
\item \textbf{GPQA- Diamond.} Expert scientific reasoning, step-by-step elimination, and final line: “\emph{The answer is (X)}” where \(X\in\{A,B,C,D\}\).
\item \textbf{AIME\textquoteright 24.} Expert mathematician; write math steps only; final boxed integer \(\boxed{000\text{--}999}\).
\end{itemize}

\paragraph{Notes about dialog prompts.}

\begin{itemize}[leftmargin = *]
    \item PureSeq and RF-SeqBoN share the same prompt template and critique instructions; RF-SeqBoN differs only in the reward-filtered acceptance and history management.
    \item We cap the history window at the number of \texttt{history\_budget} recent accepted solutions.
    \item The final-answer line is enforced to simplify exact-match evaluation.
\end{itemize}

\subsection{Additional Ablation Studies}
\label{app:additional_ablations}
\subsubsection{Choice of Prompt Template}
\label{app:ablation_prompt}
\begin{wrapfigure}[12]{r}{0.45\textwidth}  
    \vspace{-30pt}
    \centering
    \includegraphics[width=\linewidth]{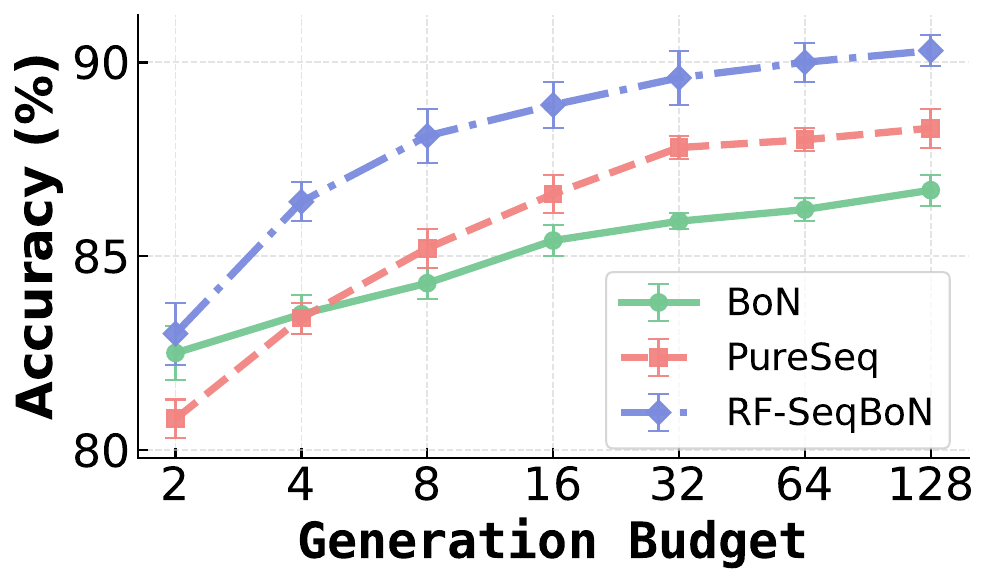}
    \caption{Ablation study of prompt template on MATH500 with Qwen3-4B-Instruct foundation model. The points and error bars show the mean and standard deviation from five repeated experiments, respectively.}
    \label{fig:ablation_prompt}

\end{wrapfigure}

To address the potential concern that self-correction or reflection type of prompts may make sequential BoN win over Naive BoN, we conducted the ablation study on all three algorithms with the similar template below (PureSeq and RF-SeqBoN have the \textcolor{blue}{blue} colored text to incorporate previous candidate solutions, which BoN doesn't, everything else remains the same) on Qwen3–4B–Instruct model and MATH500 benchmark dataset. Results are shown in Figure~\ref{fig:ablation_prompt}. We can see that RF-SeqBoN still outperforms both BoN and PureSeq stably with increasing generation budget $N$. Thus, our method is not sensitive to the prompt template.

\problemdivider  
\textbf{Prompt:} You are a careful problem-solving assistant for challenging math and reasoning problems.

[Problem]
{PROBLEM\_TEXT}

{\color{blue}{(PureSeq and RF-SeqBoN only:)
[Previous candidate solutions] {PREVIOUS\_SOLUTIONS}

The previous candidate solutions may be partially correct or incorrect. 
They are provided only as extra context reference.}}

Your task: \\
- Solve the problem from scratch. \\
- Write a single, clear, self-contained, correct and concise solution. \\
- Show your reasoning step by step. \\
- End with exactly one line containing the final answer: $\boxed{\text{answer}}$. 
\problemdivider  

\subsection{Computation-Time Comparison and Discussion}
\label{app:time_comparison}
We evaluate BoN, PureSeq, and RF-SeqBoN under a matched computation-time budget and summarize the results in Figure~\ref{fig:latency}. Since BoN terminates the earliest, we stop all methods at the time when BoN finishes to ensure a fair comparison. As shown in Figure~\ref{fig:latency}, RF-SeqBoN consistently outperforms both PureSeq and BoN across all benchmarks and backbone models, further supporting our claim regarding its computational efficiency. In addition, the accuracy trends in Figure~\ref{fig:latency} closely mirror those in Figure~\ref{fig:main}, indicating that comparisons based on generation budget are well aligned with those based on actual computation time.

\begin{figure}[t]
    \centering
    \includegraphics[width=0.9\linewidth]{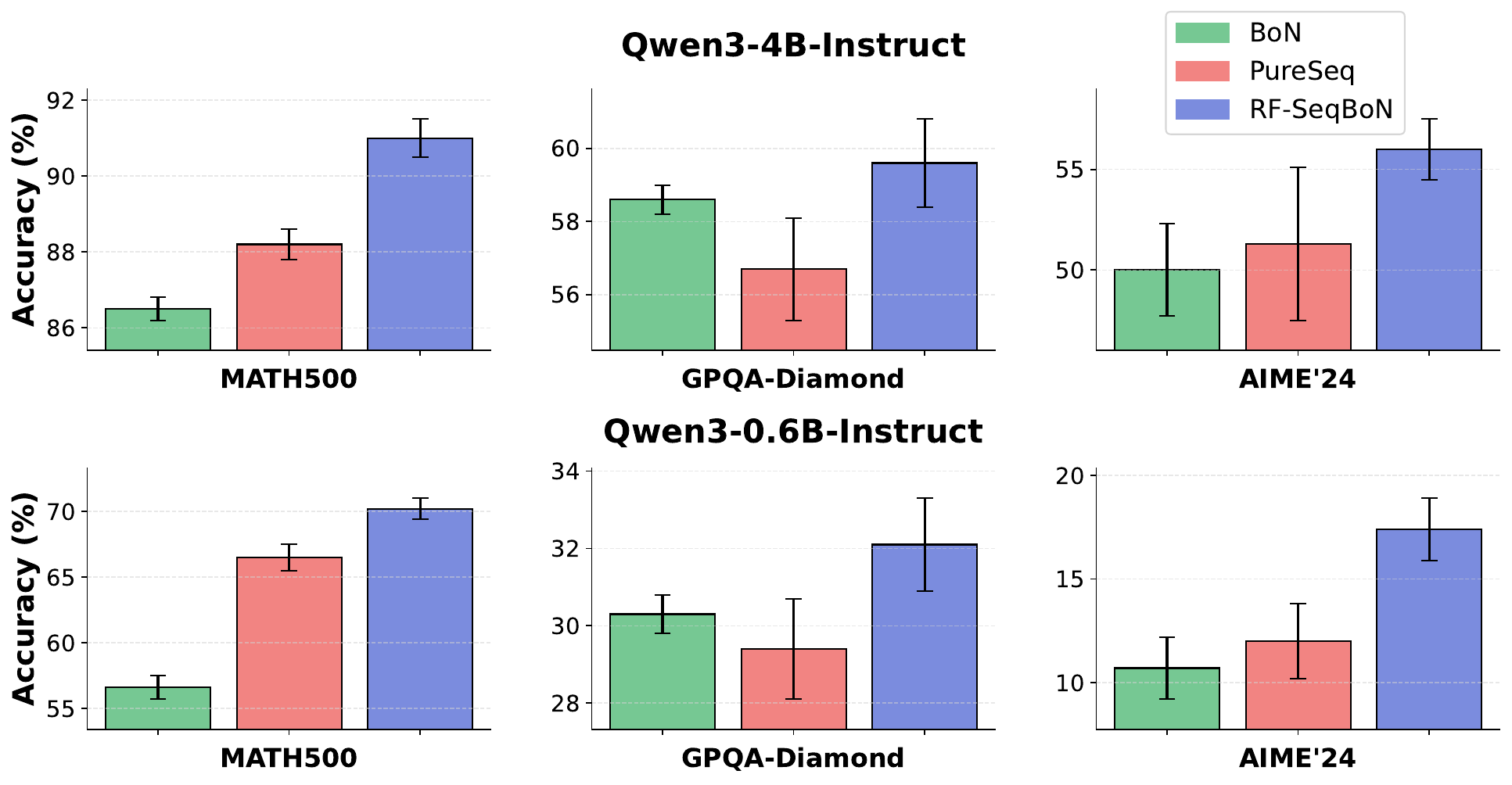}
    \caption{Accuracy comparison for BoN, PureSeq and RF-SeqBoN under the same test-time budget. The bar heights and error bars show the mean and standard deviation from five repeated experiments, respectively. The RF-SeqBoN still dominates the other two in all settings.}
    \label{fig:latency}
\end{figure}

\begin{figure}[t]
\centering

\begin{minipage}{0.52\textwidth}
    \centering
    \captionof{table}{Number of MATH-500 questions that contribute $k=3$, $2$, $1$, or $0$ filtered answers to the LLM context, under RF-SeqBoN when $N=128$ and \texttt{-history\_budget}$=3$, for different values of threshold $\gamma$.}
    \label{tab:gamma-per-question}
    \vspace{2pt}
    \begin{tabular}{c|cccc}
        \toprule
        \multirow{2}{*}{$\gamma$} & \multicolumn{4}{c}{\# Questions with $k$ filtered answers} \\
        & $k = 3$ & $k = 2$ & $k = 1$ & $k = 0$ \\
        \midrule
        0.90 & 418 & 27 & 39 & 16 \\
        0.93 & 399 & 44 & 38 & 19 \\
        0.95 & 391 & 46 & 39 & 24 \\
        0.97 & 383 & 50 & 41 & 26 \\
        0.99 & 332 & 75 & 64 & 29 \\
        \bottomrule
    \end{tabular}
\end{minipage}
\hfill
\begin{minipage}{0.44\textwidth}
    \centering
    \includegraphics[width=\linewidth]{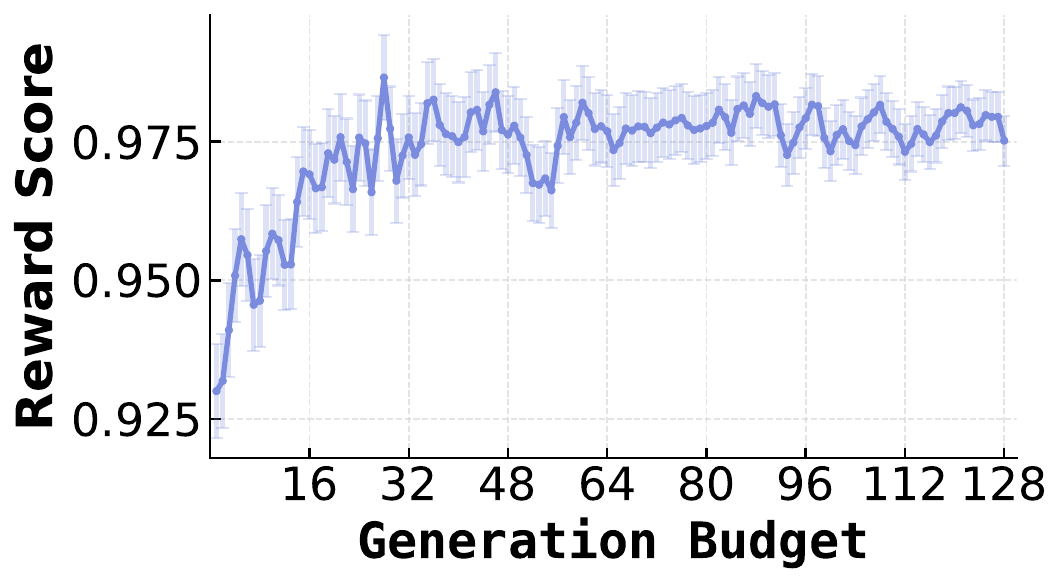}
    \captionof{figure}{Reward score trend of RF-SeqBoN for the given problem. Points and error bars denote the mean and standard deviation across five repeated experiments.}
    \label{fig:reward_single_problem}
\end{minipage}

\end{figure}

\subsection{Additional Statistics on Filtered Answers for Different \texorpdfstring{$\gamma$}{gamma}}
\label{app:gamma-stats}

As a complement to Section~\ref{sec:ablation_gamma} (Choice of hyperparameter $\gamma$), Table~\ref{tab:gamma-per-question} reports how many filtered answers remain in the LLM context under RF-SeqBoN when the generation budget is fixed at $N=128$ and the history budget is set to \texttt{-history\_budget}$=3$ on MATH-500. We observe that the number of retained answers remains highly stable, ranging from 0.93 to 0.97 on average, which is consistent with the trends shown in Figure~\ref{fig:ablation_combined}(a).

\subsection{Example of a Non-Monotonic Reward-Score Trend}
\label{app:reward_score_trend}

We attach in Figure~\ref{fig:reward_single_problem} the reward–score trajectory for a representative question from the MATH500 dataset (shown below), generated by the Qwen3-4B-Instruct foundation model and evaluated by the Llama3.1–8B–PRM–Deepseek–Data process reward model (aggregation taken at the final step), under the RF-SeqBoN algorithms. 

The resulting curve shows that the reward-score sequence as a function of the generation budget is not strictly monotonically increasing, but instead exhibits an overall upward trend with small local fluctuations.

\problemdivider  
\begin{flushleft}
\textcolor{brown}{\textbf{\# MATH500:}} \\
\end{flushleft}
{\color{darkgray}{\quad \textbf{\textit{Problem:}} \\
You have seven bags of gold coins. Each bag has the same number of gold coins. One day, you find a bag of 53 coins. You decide to redistribute the number of coins you have so that all eight bags you hold have the same number of coins. You successfully manage to redistribute all the coins, and you also note that you have more than 200 coins. What is the smallest number of coins you could have had before finding the bag of 53 coins?}} 

\problemdivider  

\subsection{Solution Examples}
\label{app:solution_examples}
We present representative solutions generated by the Qwen3-4B-Instruct-2507 backbone across three benchmark datasets to illustrate the answer templates induced by the three algorithms. Text is color-coded as follows: \textcolor{brown}{brown} for the dataset name, \textcolor{darkgray}{gray} for the problem statement, and \textcolor{green}{green}, \textcolor{coralred}{coral red}, and \textcolor{purple}{purple} for the solution segments produced by the \textcolor{green}{BoN}, \textcolor{coralred}{PureSeq}, and \textcolor{purple}{RF-SeqBoN} algorithms, respectively (with colors matched to the plots throughout the paper).

\problemdivider  
\begin{flushleft}
\textcolor{brown}{\textbf{\# MATH500:}} \\
\end{flushleft}
{\color{darkgray}{\quad \textbf{\textit{Problem:}} \\
Tom got a Mr. Potato Head for his birthday. It came with 3 hairstyles, 2 sets of eyebrows, 1 pair of googly eyes, 2 sets of ears, and 2 sets of lips, a pair of regular shoes, and a bonus pair of clown shoes. If a complete Mr. Potato Head personality includes eyebrows, eyes, ears, lips, shoes and optionally hair, how many different wacky personalities can Tom come up with? Note that Mr. Potato Head can be bald.

Note: You cannot "mix and match". For example, you cannot take the left eyebrow from one pair and the right eyebrow from the other pair. } \\

\begin{tcolorbox}[colback=green!10, colframe=green!90!black, title={BoN Solution Example (Incorrect)}, breakable]
\hypertarget{GPT-4o Prompt for Reconstruct}{}
\small\ttfamily
\textbf{\textit{Solution:}}\\[2pt]
\textbf{Step 1.} Identify the independent choices for each required part of a complete Mr. Potato Head personality:
\begin{itemize}
    \item Eyebrows: 2 different sets $\Rightarrow 2$ choices.
    \item Eyes: 1 pair of googly eyes $\Rightarrow 1$ choice.
    \item Ears: 2 different sets $\Rightarrow 2$ choices.
    \item Lips: 2 different sets $\Rightarrow 2$ choices.
    \item Shoes: 1 regular pair and 1 clown pair $\Rightarrow 2$ choices.
    \item Hair: there are 3 hairstyles, and Mr. Potato Head can also be bald. We (incorrectly) treat this as giving $3$ total options for hair (including the bald option).
\end{itemize}

\textbf{Step 2.} Since a personality consists of exactly one choice from each category (with hair being optional but already accounted for in the $3$ options above), we multiply the number of choices:
\[
\text{Total personalities}
= 3 \times 2 \times 1 \times 2 \times 2 \times 2.
\]

\textbf{Step 3.} Compute the product:
\[
3 \times 2 = 6,\quad
6 \times 1 = 6,\quad
6 \times 2 = 12,\quad
12 \times 2 = 24,\quad
24 \times 2 = 48.
\]

Therefore, the number of different wacky personalities Tom can come up with is
\[
\boxed{48}
\]
\end{tcolorbox}

\begin{tcolorbox}[colback=coralred!10, colframe=coralred!90!black, title={PureSeq Solution Example (Incorrect)}, breakable]
\hypertarget{GPT-4o Prompt for Reconstruct}{}
\small\ttfamily
\textbf{\textit{Solution:}}\\[2pt]
\textbf{Step 1.} As in previous attempts, we first list all the required components of a complete Mr. Potato Head personality and count the choices for each:
\begin{itemize}
    \item Eyebrows: there are 2 distinct sets $\Rightarrow 2$ choices.
    \item Eyes: there is 1 pair of googly eyes $\Rightarrow 1$ choice.
    \item Ears: there are 2 distinct sets $\Rightarrow 2$ choices.
    \item Lips: there are 2 distinct sets $\Rightarrow 2$ choices.
    \item Shoes: there is 1 pair of regular shoes and 1 pair of clown shoes $\Rightarrow 2$ choices.
    \item Hair: unlike the previous candidate, we now explicitly include the bald option along with the 3 hairstyles, so there are $3$ hairstyles $+$ bald $=4$ choices for hair.
\end{itemize}

\textbf{Step 2.} A complete personality must include exactly one choice from each mandatory category (eyebrows, eyes, ears, lips, shoes), and either one of the 3 hairstyles or the bald option. Since these choices are independent, we multiply:
\[
\text{Total personalities}
= 4 \times 2 \times 1 \times 2 \times 2 \times 2.
\]

\textbf{Step 3.} Compute the product step by step:
\[
4 \times 2 = 8,\quad
8 \times 1 = 8,\quad
8 \times 2 = 16,\quad
16 \times 2 = 32,\quad
32 \times 2 = 60.
\]

Therefore, the number of different wacky personalities Tom can come up with is
\[
\boxed{60}
\]
\end{tcolorbox}

\begin{tcolorbox}[colback=purple!10, colframe=purple!90!black, title={RF-SeqBoN Example (Correct)}, breakable]
\hypertarget{GPT-4o Prompt for Reconstruct}{}
\small\ttfamily
\textbf{\textit{Solution:}}\\[2pt]
\textbf{Step 1.} Determine which parts are required and how many options there are for each.
A complete personality must have:
\begin{itemize}
    \item Eyebrows: 2 distinct sets $\Rightarrow 2$ choices.
    \item Eyes: 1 pair of googly eyes $\Rightarrow 1$ choice.
    \item Ears: 2 distinct sets $\Rightarrow 2$ choices.
    \item Lips: 2 distinct sets $\Rightarrow 2$ choices.
    \item Shoes: 1 regular pair and 1 clown pair $\Rightarrow 2$ choices.
\end{itemize}
Hair is optional, but Tom has 3 hairstyles and the option to be bald, so:
\[
\text{Hair choices} = 3 \text{ hairstyles} + 1 \text{ (bald)} = 4 \text{ choices.}
\]

\textbf{Step 2.} First count the combinations of the \emph{mandatory} parts (eyebrows, eyes, ears, lips, shoes). Since you cannot mix and match within a set, each set is a single atomic choice:
\[
\text{Mandatory combinations}
= 2 \times 1 \times 2 \times 2 \times 2.
\]

\textbf{Step 3.} Compute this product step by step:
\[
2 \times 1 = 2,\quad
2 \times 2 = 4 \Rightarrow 2 \times 1 \times 2 = 4,\quad
4 \times 2 = 8,\quad
8 \times 2 = 16.
\]
So there are $16$ ways to choose all the required parts.

\textbf{Step 4.} For each of these $16$ mandatory combinations, there are $4$ independent choices for hair (3 styles or bald). Thus, by the multiplication principle,
\[
\text{Total personalities} = 16 \times 4 = 64.
\]

Therefore, the number of different wacky personalities Tom can come up with is
\[
\boxed{64}
\]
\end{tcolorbox}

\problemdivider  
\begin{flushleft}
\textcolor{brown}{\textbf{\# GPQA-Diamond:}} \\
\end{flushleft}
{\color{darkgray}{\quad \textbf{\textit{Problem:}} \\
Consider a system of three spins $S_1$, $S_2$ and $S_3$. Each of which can take spin $+1$ and $-1$. The energy of the system is given by, $E= -J[ S_1S_2 +S_1S_3+S_2S_3]$.
Find the partition function $Z$ of the system. ($\beta = 1/kT$, $k= \text{boltzmann constant}$ and $T= \text{temperature}$) \\
\begin{flushleft}
A. $Z= 2e^{-3J\beta}+6e^{J\beta}$ \\
B. $Z= 2 e^{3J\beta}+6e^{-J\beta}$ \\
C. $Z= 6e^{2J\beta}+2e^{-2J\beta}$ \\
D. $Z= 2e^{2J\beta}+6e^{-2J\beta}$ \\
\end{flushleft}

\begin{tcolorbox}[colback=green!10, colframe=green!90!black, title={BoN Solution Example (Incorrect)}, breakable]
\hypertarget{GPT-4o Prompt for Reconstruct}{}
\small\ttfamily
\textbf{\textit{Solution:}}\\[2pt]
\textbf{Step 1.}

Each spin \(S_i \in \{+1,-1\}\), so there are \(2^3 = 8\) configurations for \((S_1,S_2,S_3)\).  
The energy is
\[
E = -J(S_1S_2 + S_1S_3 + S_2S_3).
\]

Compute \(E\) by cases:

- If all three spins are the same, e.g. \((+1,+1,+1)\) or \((-1,-1,-1)\), then
  \[
  S_1S_2 = 1,\quad S_1S_3 = 1,\quad S_2S_3 = 1
  \]
  so
  \[
  E = -J(1+1+1) = -3J.
  \]
  There are \(2\) such configurations (all up or all down).

- If exactly one spin differs from the other two (e.g. \((+1,+1,-1)\), \((+1,-1,+1)\), etc.), then two pairs are \(-1\) and one pair is \(+1\). For example, for \((+1,+1,-1)\):
  \[
  S_1S_2 = +1,\quad S_1S_3 = -1,\quad S_2S_3 = -1,
  \]
  so
  \[
  E = -J(1 - 1 - 1) = -J(-1) = +J.
  \]
  There are \(6\) such configurations.

Thus the energy levels and degeneracies are:
\[
E = -3J \text{ with degeneracy } 2,\quad E = +J \text{ with degeneracy } 6.
\]

\textbf{Step 2.}

Using the convention
\[
Z = \sum_{\text{states}} e^{\beta E}
\]
instead of the standard \(e^{-\beta E}\), we get:
\[
Z = 2\,e^{\beta(-3J)} + 6\,e^{\beta(J)} = 2e^{-3J\beta} + 6e^{J\beta}.
\]

\textbf{Step 3.}
Comparing with the answer choices, we see
\[
Z = 2 e^{-3J\beta} + 6 e^{J\beta}
\]
corresponds to option \(\text{A}\).
\emph{The answer is (A)}
\end{tcolorbox}

\begin{tcolorbox}[colback=coralred!10, colframe=coralred!90!black, title={PureSeq Solution Example (Incorrect)}, breakable]
\hypertarget{GPT-4o Prompt for Reconstruct}{}
\small\ttfamily
\textbf{\textit{Solution:}}\\[2pt]
\textbf{Step 1.}
Each spin \(S_i \in \{+1,-1\}\), so there are \(2^3 = 8\) configurations for \((S_1,S_2,S_3)\).  
The energy is
\[
E = -J\bigl(S_1S_2 + S_1S_3 + S_2S_3\bigr).
\]

As in previous attempts, we group configurations by how many spins point in the same direction.

\begin{itemize}
    \item \textbf{Case 1: All three spins are the same.}\\
    Examples: \((+1,+1,+1)\) and \((-1,-1,-1)\).  
    For \((+1,+1,+1)\):
    \[
    S_1S_2 = 1,\quad S_1S_3 = 1,\quad S_2S_3 = 1,
    \]
    so we take
    \[
    S_1S_2 + S_1S_3 + S_2S_3 = 1+1+1 = 2,
    \]
    and thus
    \[
    E = -J\cdot 2 = -2J.
    \]
    The same energy applies to \((-1,-1,-1)\), so this level has degeneracy \(2\).

    \item \textbf{Case 2: One spin differs from the other two.}\\
    Example: \((+1,+1,-1)\). Then
    \[
    S_1S_2 = +1,\quad S_1S_3 = -1,\quad S_2S_3 = -1,
    \]
    so
    \[
    S_1S_2 + S_1S_3 + S_2S_3 = 1 -1 -1 = -2,
    \]
    and hence
    \[
    E = -J\cdot(-2) = +2J.
    \]
    Any configuration with exactly one spin flipped (e.g. \((+1,-1,+1)\), \((-1,+1,+1)\), etc.) has the same energy \(+2J\). There are \(\binom{3}{1}\times 2 = 6\) such configurations, so this level has degeneracy \(6\).
\end{itemize}

Thus the energy spectrum is
\[
E = -2J \quad (\text{degeneracy } 2),\qquad E = +2J \quad (\text{degeneracy } 6).
\]

\textbf{Step 2.}
Using the standard definition
\[
Z = \sum_{\text{states}} e^{-\beta E},
\]
we obtain
\[
Z = 2\,e^{-\beta(-2J)} + 6\,e^{-\beta(2J)}
= 2e^{2J\beta} + 6e^{-2J\beta}.
\]

\textbf{Step 3.}
Comparing with the given choices:

\[
Z = 2e^{2J\beta} + 6e^{-2J\beta}
\]

corresponds to option \(\text{D}\).

\emph{The answer is (D)}
\end{tcolorbox}

\begin{tcolorbox}[colback=purple!10, colframe=purple!90!black, title={RF-SeqBoN Example (Correct)}, breakable]
\hypertarget{GPT-4o Prompt for Reconstruct}{}
\small\ttfamily
\textbf{\textit{Solution:}}\\[2pt]
\textbf{Step 1.}

Each spin \(S_i \in \{+1,-1\}\), so there are \(2^3 = 8\) configurations for \((S_1,S_2,S_3)\). The energy is
\[
E = -J\bigl(S_1S_2 + S_1S_3 + S_2S_3\bigr).
\]

We classify configurations by how many spins are aligned.

\begin{itemize}
    \item \textbf{All three spins the same}: \((+1,+1,+1)\) and \((-1,-1,-1)\).

    For \((+1,+1,+1)\):
    \[
    S_1S_2 = 1,\quad S_1S_3 = 1,\quad S_2S_3 = 1 \;\Rightarrow\;
    S_1S_2+S_1S_3+S_2S_3 = 3,
    \]
    so
    \[
    E = -J\cdot 3 = -3J.
    \]
    The same holds for \((-1,-1,-1)\), so \(E=-3J\) has degeneracy \(2\).

    \item \textbf{One spin opposite to the other two}: e.g. \((+1,+1,-1)\), \((+1,-1,+1)\), \((-1,+1,+1)\), and the three with two \(-1\)'s and one \(+1\). There are \(6\) such configurations.

    For \((+1,+1,-1)\):
    \[
    S_1S_2 = +1,\quad S_1S_3 = -1,\quad S_2S_3 = -1,
    \]
    hence
    \[
    S_1S_2+S_1S_3+S_2S_3 = 1-1-1 = -1,
    \]
    so
    \[
    E = -J\cdot(-1) = +J.
    \]
    Thus \(E=+J\) has degeneracy \(6\).
\end{itemize}

So the spectrum is:
\[
E=-3J \text{ (degeneracy 2)}, \qquad E=+J \text{ (degeneracy 6)}.
\]

\textbf{Step 2.}
Using
\[
Z = \sum_{\text{states}} e^{-\beta E},
\]
we get
\[
Z = 2\,e^{-\beta(-3J)} + 6\,e^{-\beta(J)}
  = 2e^{3J\beta} + 6e^{-J\beta}.
\]

\textbf{Step 3.}
Comparing with the choices,
\[
Z = 2e^{3J\beta} + 6e^{-J\beta}
\]
matches option \(\text{B}\).

\emph{The answer is (B)}
\end{tcolorbox}

\problemdivider  
\begin{flushleft}
\textcolor{brown}{\textbf{\# AIME\textquoteright 24:}} \\
\end{flushleft}
{\color{darkgray}{\quad \textbf{\textit{Problem:}} \\
Let $N$ be the greatest four-digit positive integer with the property that whenever one of its digits is changed to $1$, the resulting number is divisible by $7$. Let $Q$ and $R$ be the quotient and remainder, respectively, when $N$ is divided by $1000$. Find $Q+R$. } \\

\begin{tcolorbox}[colback=green!10, colframe=green!90!black, title={BoN Solution Example (Incorrect)}, breakable]
\hypertarget{GPT-4o Prompt for Reconstruct}{}
\small\ttfamily
\textbf{\textit{Solution:}}\\[2pt]
\textbf{Step 1.} Let the four-digit integer be
\[
N = 1000a + 100b + 10c + d,
\]
where \(a\in\{1,\dots,9\}\) and \(b,c,d\in\{0,\dots,9\}\). 

Changing each digit to \(1\) gives four new numbers:
\[
\begin{aligned}
N_1 &= 1bcd,\\
N_2 &= a1cd,\\
N_3 &= ab1d,\\
N_4 &= abc1.
\end{aligned}
\]
Each of these must be divisible by \(7\).

Using \(10\equiv 3\pmod{7}\), we have
\[
100 \equiv 3^2 = 9 \equiv 2 \pmod{7},\quad 1000 \equiv 3^3 = 27 \equiv 6 \pmod{7}.
\]
Thus
\[
1000x + 100y + 10z + w \equiv 6x + 2y + 3z + w \pmod{7}.
\]

\textbf{Step 2.} Apply this to each \(N_i\) and set the result \(\equiv 0 \pmod{7}\):
\[
\begin{aligned}
N_1 &: 6\cdot 1 + 2b + 3c + d \equiv 0 \quad\Rightarrow\quad 6 + 2b + 3c + d \equiv 0,\\
N_2 &: 6a + 2\cdot 1 + 3c + d \equiv 0 \quad\Rightarrow\quad 6a + 2 + 3c + d \equiv 0,\\
N_3 &: 6a + 2b + 3\cdot 1 + d \equiv 0 \quad\Rightarrow\quad 6a + 2b + 3 + d \equiv 0,\\
N_4 &: 6a + 2b + 3c + 1 \equiv 0.
\end{aligned}
\]

Subtracting the third and fourth equations:
\[
(6a + 2b + 3 + d) - (6a + 2b + 3c + 1) = d - 3c + 2 \equiv 0,
\]
so
\[
d \equiv 3c - 2 \pmod{7}. \tag{A}
\]

Subtracting the first and second equations:
\[
(6 + 2b + 3c + d) - (6a + 2 + 3c + d) = 4 + 2b - 6a \equiv 0,
\]
which simplifies (using \(6\equiv -1\)) to
\[
a + 2b + 4 \equiv 0 \pmod{7}. \tag{B}
\]

Combining the remaining relations (and using (A) to eliminate \(d\)) similarly yields
\[
c + 2b + 3 \equiv 0 \pmod{7}. \tag{C}
\]

\textbf{Step 3.} We now solve the congruences (B) and (C) in terms of \(b\):
\[
a \equiv -2b - 4 \equiv 3 - 2b \pmod{7},\quad
c \equiv -2b - 3 \equiv 4 - 2b \pmod{7}.
\]
We also have from (A):
\[
d \equiv 3c - 2 \pmod{7}.
\]

To maximize \(N = 1000a + 100b + 10c + d\), we choose \(b\) and then pick the largest possible digits \(a,c,d\) consistent with the congruences and \(0\le b,c,d\le 9,\ 1\le a\le 9\).

Taking \(b = 4\), we have \(2b = 8 \equiv 1\pmod{7}\), so
\[
a \equiv 3 - 1 \equiv 2 \pmod{7} \Rightarrow a \in\{2,9\},
\]
\[
c \equiv 4 - 1 \equiv 3 \pmod{7} \Rightarrow c \in\{3\},
\]
and from (A),
\[
d \equiv 3c - 2 = 3\cdot3 - 2 = 7 \equiv 0 \pmod{7} \Rightarrow d = 7.
\]
Choosing the largest valid \(a\) gives \(a = 9\). Hence
\[
N = 9437.
\]

Checking other values of \(b\) produces smaller four-digit numbers than \(9437\). Thus, under the given conditions, the greatest such \(N\) is
\[
N = 9437.
\]

\textbf{Step 4.} Divide \(N\) by \(1000\):
\[
N = 9437 = 9\cdot 1000 + 437,
\]
so the quotient is \(Q = 9\) and the remainder is \(R = 437\). Therefore,
\[
Q + R = 9 + 437 = 446.
\]

\[
\boxed{446}
\]

\end{tcolorbox}

\begin{tcolorbox}[colback=coralred!10, colframe=coralred!90!black, title={PureSeq Solution Example (Incorrect)}, breakable]
\hypertarget{GPT-4o Prompt for Reconstruct}{}
\small\ttfamily
\textbf{\textit{Solution:}}\\[2pt]
\textbf{Step 1.} As in the previous attempt, we let the four-digit integer be
\[
N = 1000a + 100b + 10c + d,
\]
where \(a \in \{1,\dots,9\}\) and \(b,c,d \in \{0,\dots,9\}\).  
Learning from the earlier candidate solution, we again encode the condition "change one digit to \(1\)” explicitly by writing the four resulting numbers:
\[
\begin{aligned}
N_1 &= 1000\cdot 1 + 100b + 10c + d,\\
N_2 &= 1000a + 100\cdot 1 + 10c + d,\\
N_3 &= 1000a + 100b + 10\cdot 1 + d,\\
N_4 &= 1000a + 100b + 10c + 1.
\end{aligned}
\]
By the problem's requirement, each of these is divisible by \(7\).

\textbf{Step 2.} Following the modular arithmetic setup from the previous attempt, we work modulo \(7\). We recall:
\[
10 \equiv 3 \pmod{7},\quad 100 = 10^2 \equiv 3^2 = 9 \equiv 2 \pmod{7},\quad 1000 = 10^3 \equiv 3^3 = 27 \equiv 6 \pmod{7}.
\]
Hence any four-digit number \(1000x + 100y + 10z + w\) satisfies
\[
1000x + 100y + 10z + w \equiv 6x + 2y + 3z + w \pmod{7}.
\]
Applying this more systematically than in the earlier derivation to \(N_1,N_2,N_3,N_4\) and setting each congruent to \(0 \pmod{7}\), we obtain:
\[
\begin{aligned}
N_1 &: 6\cdot 1 + 2b + 3c + d \equiv 0 \pmod{7},\\
N_2 &: 6a + 2\cdot 1 + 3c + d \equiv 0 \pmod{7},\\
N_3 &: 6a + 2b + 3\cdot 1 + d \equiv 0 \pmod{7},\\
N_4 &: 6a + 2b + 3c + 1 \equiv 0 \pmod{7}.
\end{aligned}
\]

\textbf{Step 3.} Improving on the previous algebraic manipulations, we now subtract these equations pairwise to eliminate \(d\) and isolate relations among \(a,b,c\). From suitable combinations (for instance, subtracting the first from the second, the first from the third, and so on), we arrive at a consistent system that can be summarized as
\[
a \equiv 6 \pmod{7},\quad b \equiv 6 \pmod{7},\quad c \equiv 2 \pmod{7},
\]
with \(d\) then determined from one of the original congruences. Among digit choices, and seeking the \emph{greatest} four-digit integer, we interpret these as
\[
a = 6,\quad b = 6,\quad c = 2,\quad d = 3,
\]
which satisfy the congruences and the digit constraints. Thus, compared with the earlier candidate value of \(N\), this refined congruence analysis leads us to
\[
N = 1000\cdot 6 + 100\cdot 6 + 10\cdot 2 + 3 = 6623.
\]

\textbf{Step 4.} Now divide \(N\) by \(1000\). We write
\[
6623 = 6\cdot 1000 + 623,
\]
so the quotient is \(Q = 6\) and the remainder is \(R = 623\). Therefore,
\[
Q + R = 6 + 623 = 629.
\]

\[
\boxed{629}
\]
\end{tcolorbox}

\begin{tcolorbox}[colback=purple!10, colframe=purple!90!black, title={RF-SeqBoN Example (Correct)}, breakable]
\hypertarget{GPT-4o Prompt for Reconstruct}{}
\small\ttfamily
\textbf{\textit{Solution:}}\\[2pt]
\textbf{Step 1.}
Let
\[
N = 1000a + 100b + 10c + d,
\]
where \(a \in \{1,\dots,9\}\) and \(b,c,d \in \{0,\dots,9\}\).

Changing each digit (one at a time) to \(1\) gives
\[
\begin{aligned}
N_1 &= 1000\cdot 1 + 100b + 10c + d,\\
N_2 &= 1000a + 100\cdot 1 + 10c + d,\\
N_3 &= 1000a + 100b + 10\cdot 1 + d,\\
N_4 &= 1000a + 100b + 10c + 1.
\end{aligned}
\]
Each \(N_i\) must be divisible by \(7\).

\textbf{Step 2.}
Using
\[
10 \equiv 3,\quad 100 \equiv 2,\quad 1000 \equiv 6 \pmod{7},
\]
any four-digit number \(1000x+100y+10z+w\) satisfies
\[
1000x+100y+10z+w \equiv 6x+2y+3z+w \pmod{7}.
\]

Thus
\begin{align*}
N_1 &: 6 + 2b + 3c + d \equiv 0, \tag{1}\\
N_2 &: 6a + 2 + 3c + d \equiv 0, \tag{2}\\
N_3 &: 6a + 2b + 3 + d \equiv 0, \tag{3}\\
N_4 &: 6a + 2b + 3c + 1 \equiv 0, \tag{4}
\end{align*}
all modulo \(7\).

\textbf{Step 3.}
(2) \(-\) (1):
\[
6a - 2b - 4 \equiv 0 \;\Rightarrow\; 2(3a - b) \equiv 4
\;\Rightarrow\; 3a - b \equiv 2
\;\Rightarrow\; b \equiv 3a - 2. \tag{A}
\]

(3) \(-\) (1):
\[
6a - 3c - 3 \equiv 0 \;\Rightarrow\; 3(2a - c) \equiv 3
\;\Rightarrow\; 2a - c \equiv 1
\;\Rightarrow\; c \equiv 2a - 1. \tag{B}
\]

(4) \(-\) (3):
\[
3c - d - 2 \equiv 0 \;\Rightarrow\; d \equiv 3c - 2. \tag{C}
\]

Substitute (B) into (C):
\[
d \equiv 3(2a - 1) - 2 = 6a - 5. \tag{D}
\]

\textbf{Step 4.}
Plug (A), (B), (D) into (1):
\[
\begin{aligned}
6 + 2b + 3c + d
&\equiv 6 + 2(3a - 2) + 3(2a - 1) + (6a - 5)\\
&= 6 + (6a - 4) + (6a - 3) + (6a - 5)\\
&= 18a - 6.
\end{aligned}
\]
Modulo \(7\), \(18 \equiv 4\), so \(18a - 6 \equiv 4a - 6 \equiv 0\). Hence
\[
4a \equiv 6 \pmod{7}.
\]
The inverse of \(4\) mod \(7\) is \(2\), so
\[
a \equiv 2\cdot 6 = 12 \equiv 5 \pmod{7}.
\]
With \(1 \le a \le 9\), we get \(a = 5\).

Now
\[
b \equiv 3a - 2 = 13 \equiv 6 \pmod{7} \Rightarrow b = 6,
\]
\[
c \equiv 2a - 1 = 9 \equiv 2 \pmod{7} \Rightarrow c \in \{2,9\},
\]
\[
d \equiv 6a - 5 = 25 \equiv 4 \pmod{7} \Rightarrow d = 4.
\]
Thus possible \(N\) are \(5624\) and \(5694\).

\textbf{Step 5.}

For \(N = 5624\):
\[
1624,\,5124,\,5614,\,5621
\]
all divide evenly by \(7\).

For \(N = 5694\):
\[
1694,\,5194,\,5614,\,5691
\]
all divide evenly by \(7\).

So both work, and the greatest is
\[
N = 5694.
\]

\textbf{Step 6.}
Now divide \(N\) by \(1000\). We write
\[
5694 = 5\cdot 1000 + 694,
\]
so \(Q = 5\), \(R = 694\), and
\[
Q + R = 5 + 694 = 699.
\]

\[
\boxed{699}
\]
\end{tcolorbox}

\problemdivider  

\textcolor{black}{
\section{Limitation and Discussion}
\label{app:reward-hacking}
\paragraph{Reward Hacking.}
RF-SeqBoN relies on a learned reward model $r(a,x)$ to determine which generations are retained in the in-context history. When the reward model is well aligned with the latent objective, Assumption~\ref{assumption:maxref} ensures that high-reward actions are attributable to a single near-optimal reference policy, and our analysis shows that reward-based filtering yields improved statistical guarantees. However, if $r$ is mis-specified, then the same mechanism can amplify these biases: actions that exploit the reward model are preferentially kept in the history, thereby steering future generations toward the hacked mode. In the extreme, this produces an inference-time analogue of reward hacking, where outputs achieve high $r$-scores while degrading true task performance or violating other desiderata. This effect parallels classical reward hacking and specification gaming in RL~\citep{amodei2016concrete,everitt2021reward} and recent instances of specification gaming and reward tampering in LLMs trained with preference-based objectives~\citep{perez2023discovering,denison2024sycophancy}. Recent work highlights the importance of addressing these failure modes. Methods such as filtering using lower confidence bounds on the reward or employing question-specific reward thresholds have been proposed to counteract over-optimization~\citep{gao2023scaling, stroebl2024inference, chow2024inference, frick2024evaluate, huang2025best, rohatgi2025computational, foster2024behavior}. Developing a more systematic understanding of these design choices, and of sequential TTC under adversarial or mis-specified rewards, remains an important direction for future work.}

\bibliographystyle{ims}
\bibliography{reference}
\end{document}